%% file: arxiv_main.tex
\pdfoutput=1
\documentclass[11pt,onecolumn]{cleantechnicalreport}

\usepackage[authoryear,sort&compress,round]{natbib}
\usepackage{hybridfrontmatter}
\usepackage{wrapstuff}
\usepackage{placeins}
\input{math_commands.tex}

\usepackage{algorithm}
\usepackage{algpseudocode}

\providecommand{\sg}[1]{\operatorname{sg}\!\left[#1\right]}
\providecommand{\E}{\mathbb{E}}
\providecommand{\Cov}{\operatorname{Cov}}
\providecommand{\avg}[1]{\frac{1}{|\mathcal T|}\sum_{t\in\mathcal T} #1}

\newtheorem{proposition}{Proposition}
\newtheorem{lemma}{Lemma}
\newtheorem{theorem}{Theorem}

\newtheorem{assumption}{Assumption}

\newenvironment{Lemma}{\begin{lemma}}{\end{lemma}}

\newcommand{\Method}{SCALE}

\graphicspath{{}}

\renewcommand{\ReportTitleFont}{%
  \normalfont\bfseries\fontsize{14}{17}\selectfont}

\makeatletter
\fancypagestyle{hybridfirststyle}{%
  \fancyhf{}%
  \fancyfoot[L]{%
    \ifdefempty{\report@footertext}{}{%
      {\color{black}\normalfont\fontsize{8}{10}\selectfont\report@footertext}%
    }%
  }%
  \fancyfoot[C]{}%
  \renewcommand{\footrule}{\color{gray}\DefaultFootRule}%
}
\makeatother

\hypersetup{
  colorlinks=true,
  linkcolor=blue,
  citecolor=blue,
  urlcolor=blue,
  pdftitle={Rethinking Token Reweighting for SFT: Suppress, Reverse, and Extrapolate Learned Features},
  pdfauthor={Anonymous Authors}
}

\let\cite\citep
\title{Rethinking Token Reweighting for SFT: \\Suppress, Reverse, and Extrapolate Learned Features}

\author{%
\begin{minipage}{0.96\textwidth}
\vspace*{0.3em}
\centering
\normalfont
{\small\textbf{Cunchun Li}\textsuperscript{1,2,*}\enspace
\textbf{Haonan He}\textsuperscript{1,2,*,$\dagger$}\enspace
\textbf{Yifan Gao}\textsuperscript{1,2}\enspace
\textbf{Minglei Li}\textsuperscript{1,3}\enspace \\
\textbf{Jingqi Ye}\textsuperscript{1,2}\enspace
\textbf{Qingyu Yang}\textsuperscript{1,4}\enspace 
\textbf{Peng Ye}\textsuperscript{1,5,$\ddagger$}\enspace
}
\end{minipage}%
}

\affil{%
\textsuperscript{1}\textbf{Shanghai AI Laboratory} \quad
\textsuperscript{2}\textbf{University of Science and Technology of China} \\
\textsuperscript{3}\textbf{Fudan University} \quad
\textsuperscript{4}\textbf{KTH Royal Institute of Technology} \quad
\textsuperscript{5}\textbf{The Chinese University of Hong Kong} \\
{\footnotesize
\textsuperscript{*}Equal Contribution,\quad
\textsuperscript{\textdagger}Project Lead,\quad
\textsuperscript{$\ddagger$}Corresponding Author \\
\faEnvelope\hspace{0.35em}\texttt{yepeng@pjlab.org.cn}
}
}

\begin{abstract}
Supervised fine-tuning (SFT) learns most aggressively from tokens that the model deems least likely. This helps acquire new behaviors, but also amplifies noisy or conflicting supervision and can overwrite useful pretrained knowledge. Through a unified policy-loss view, we revisit existing token-reweighting methods and show that they assign nonnegative coefficients to demonstrated tokens. Consequently, they can suppress or amplify supervised updates, but cannot reverse harmful features once learned. Moreover, larger training weights do not amount to feature extrapolation, since they change the optimization trajectory rather than scale a fixed SFT direction. We argue that reversal and extrapolation require a stable reference frame defined by a fixed SFT delta. Motivated by this, we propose \Method{} (\textbf{S}elective \textbf{C}ontrol of \textbf{A}daptation via \textbf{L}ocal \textbf{E}ntropy), an entropy-guided adaptation-strength-control method that freezes the pretrained model and the SFT delta and learns bounded token- and module-specific gates by minimizing predictive entropy alone. These gates suppress, reverse, or extrapolate frozen SFT features according to their alignment with entropy reduction. Across Qwen2.5-Math-1.5B, Qwen2.5-Math-7B, and Qwen3-4B-Base, \Method{} achieves mathematical-reasoning averages of \(37.84\), \(43.60\), and \(36.57\), exceeding the strongest corresponding baselines while remaining competitive on general-retention benchmarks. It also attains the best average code-generation performance across HumanEval, HumanEval+, and MBPP for all three models. These results suggest that effective SFT correction can benefit from controlling how already learned residuals are used, rather than only modifying how they are learned.
\end{abstract}

\begin{document}

\vspace*{-\headheight}
\maketitle
\enlargethispage{40pt}

\begin{center}
    \includegraphics[width=0.92\linewidth]{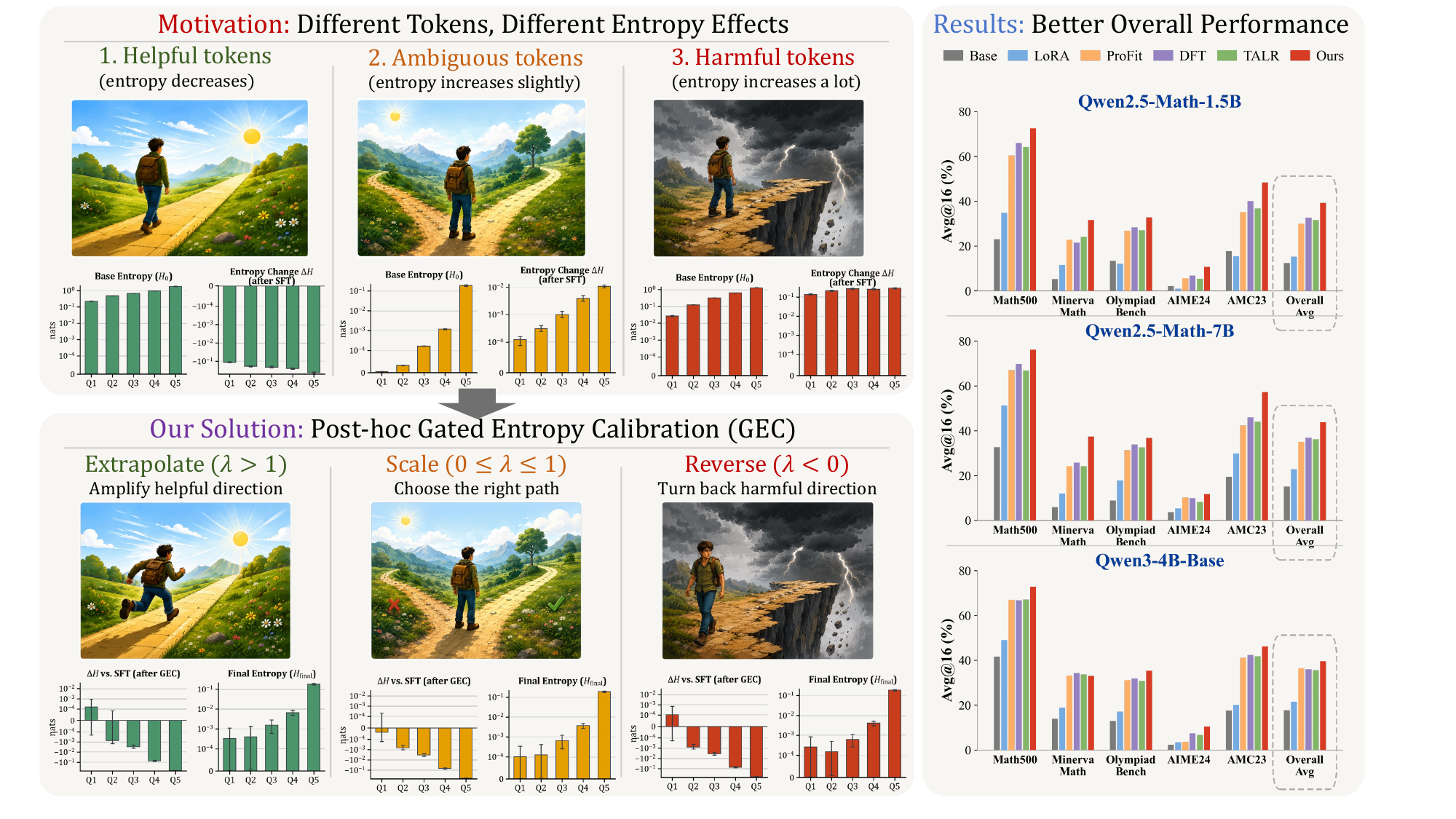}
\captionsetup{font=small,skip=6pt}
\captionof{figure}{
Motivation, illustration and performance of \Method{}.
The upper left panels show that different tokens induce distinct entropy changes
after SFT, motivating our method. The lower left panels illustrate how the learned gates exploit
this signal. The twelve bar plots summarize entropy changes before and after gating on held-out examples, providing empirical evidence. The right panels report the resulting improvements across model scales.
}
\label{fig:overview}
\end{center}

\newpage

\input{Sections/Introduction}
\input{Sections/Related_Work}
\input{Sections/Revisiting}
\input{Sections/Method}
\input{Sections/Experiment}

\section{Conclusion}

Supervised fine-tuning leaves a delta with both useful and harmful effects.
\Method{} freezes that delta and learns an entropy-guided signed gate to
suppress, reverse, or extrapolate it. Across three backbones, this control
improves mathematical reasoning and code generation while retaining general
capability, showing the value of calibrating a completed adaptation.

\clearpage
\bibliography{main}

\clearpage
\appendix
\input{Sections/Appendix}

\end{document}

%% file: math_commands.tex
\usepackage{amsmath,amsfonts,bm}

\def\eqref#1{equation~\ref{#1}}

\def\1{\bm{1}}

\DeclareMathAlphabet{\mathsfit}{\encodingdefault}{\sfdefault}{m}{sl}
\SetMathAlphabet{\mathsfit}{bold}{\encodingdefault}{\sfdefault}{bx}{n}

\newcommand{\E}{\mathbb{E}}

\newcommand{\Avg}{\mathrm{Avg@16}}

\newcommand{\Qwen}{\textsc{Qwen2.5-Math}}

\newcommand{\Cov}{\mathrm{Cov}}


%% file: Sections/Introduction.tex
\section{Introduction}

Post-training pipelines of large language models (LLMs)~\citep{zeng2026glm-5,yang2025qwen3,xu2026deepseek-v4} commonly combine supervised fine-tuning (SFT) on curated demonstrations with subsequent exploration through reinforcement learning with verifiable rewards (RLVR)~\citep{yang2025qwen3} and multi-domain integration with on-policy distillation (OPD)~\citep{agarwal2024policy}. Within this pipeline, SFT provides an essential transition from the general next-token prediction capability acquired during pretraining to crafted task-specific behaviors, structured reasoning patterns, and instruction-following capabilities~\citep{wei2022finetuned,ouyang2022training,chung2024scaling}. Despite this central role, the SFT objective remains remarkably simple: it minimizes the average negative log-likelihood of all demonstrated tokens. Let \(c_t=(x,y_{<t})\) denote the prefix context and \(y_t\) the demonstrated token. Standard SFT optimizes \(\mathcal{L}_{\mathrm{SFT}}(\theta)=-\sum_t\log\pi_\theta(y_t\mid c_t)\), whose reference-token logit gradient is \(\pi_\theta(y_t\mid c_t)-1\). Consequently, the correction toward a demonstrated token becomes stronger as the probability assigned to that token decreases, while tokens that the model already predicts confidently receive relatively small updates.

This confidence-dependent correction is a major source of both the effectiveness and the fragility of SFT~\citep{jin2025rl,zhou2023limaalignment,chu2025sftmemorizesrlgeneralizes}. It allows the model to rapidly acquire useful patterns that are unlikely to be predicted by the base model, but it also conflates informative novelty with unreliable supervision. Natural language supervision is inherently one-to-many, and a single demonstration may combine essential reasoning with interchangeable lexical choices, formatting conventions, stylistic preferences, and occasional factual or logical errors. A low-probability reference token can therefore represent two qualitatively different cases. The model may be uncertain, and the reference may provide valuable new information, or the model may confidently prefer a coherent alternative, and the reference may be noisy. Strong correction is desirable in the former case, but in the latter case, it can overwrite useful pretrained knowledge and contribute to suboptimal performance. SFT nevertheless treats both cases identically, assigning its strongest corrections precisely to the tokens whose reliability is most ambiguous.


A growing body of work addresses this problem by removing or reducing the influence of potentially harmful tokens. Data-side approaches identify low-value tokens before training using influence scores or perturbation-based saliency~\citep{pang2025token}; training-side approaches instead reweight each token's contribution to the loss. These strategies differ in when they intervene and in the reliability signal they use. Representative methods include DFT~\citep{wu2026generalizationsftreinforcementlearning}, which assigns a soft weight from the model probability of the reference token; ProFit~\citep{liu2026profitleveraginghighvaluesignals}, which masks low-probability tokens to reduce overfitting to a single reference; and EAFT~\citep{diao2026entropyadaptivefinetuningresolvingconfident}, which combines reference probability with predictive entropy to attenuate confident conflicts. Despite these differences, these representative methods share one structure: a nonnegative coefficient \(w_t\) on the demonstrated token's loss, with standard SFT setting \(w_t=1\) and data filtering the special case \(w_t\in\{0,1\}\). As our analysis in Section~\ref{sec:rw-correction} shows, this common structure also has a shared limitation. A nonnegative token weight controls the magnitude of the demonstrated-token supervised gradient; it can suppress or amplify this direct CE component but cannot reverse its sign.

Furthermore, token weights act while the task update itself is still being learned: changing
$w_t$ changes the optimization trajectory rather than scaling a fixed post-SFT delta. Reversal and extrapolation become well defined only after SFT, when the pretrained model and the learned delta are frozen to form a stable reference frame. For a gated delta $W_\ell h_{\ell,t}+\lambda_{\ell,t}\Delta W_\ell h_{\ell,t}$,
values $0\leq\lambda_{\ell,t}<1$ suppress its contribution,
$\lambda_{\ell,t}<0$ reverse it, and $\lambda_{\ell,t}>1$ extrapolate it beyond its original SFT strength. We use predictive entropy to decide how these frozen deltas should be scaled. Minimizing entropy induces the signed full-action advantage $A_t^{\mathrm{ent}}(a)=\log\pi(a\mid c_t)+H_t$, providing directional credit over the model's predictive distribution, while \Method{} restricts this signal
to gate-only adaptation over already learned SFT deltas. Under an alignment condition between entropy advantage and task value within this restricted gate subspace, entropy descent locally strengthens task-beneficial delta directions and weakens task-harmful ones. Predictive entropy is therefore not
treated as a correctness oracle, but as an unsupervised local selection signal for a bounded family of frozen delta interventions.

Building on these observations, we propose \Method{}, an entropy-guided
adaptation-strength-control method for SFT (Figure~\ref{fig:overview}). We first train a LoRA adapter using the standard SFT objective and then freeze both the pretrained model and the learned adapter. We introduce lightweight token- and module-specific gates \(\lambda_{\ell,t}\in[\lambda_{\min},\lambda_{\max}]\) with \(\lambda_{\min}<0\) and \(\lambda_{\max}>1\), modulating each learned LoRA delta as \(u_{\ell,t}=W_\ell h_{\ell,t}+\lambda_{\ell,t}\Delta_{\phi^\star,\ell}(h_{\ell,t})\). The value \(\lambda_{\ell,t}=1\) preserves the original SFT delta, values in \([0,1)\) suppress it, negative values reverse it, and values above one extrapolate it beyond its original strength. The gates are optimized exclusively by minimizing the predictive entropy of the gated model, while all pretrained and SFT parameters remain fixed. We evaluate \Method{} on mathematical reasoning, general retention, and code generation using Qwen2.5-Math-1.5B, Qwen2.5-Math-7B, and Qwen3-4B-Base. It achieves mathematical-reasoning Avg@16 averages of \(37.84\%\), \(43.60\%\), and \(36.57\%\), exceeding DFT by \(5.03\), \(6.08\), and \(1.91\) points; remains competitive on general retention; and obtains the best code-generation averages across HumanEval, HumanEval+, and MBPP for all three models (\(34.72\%\), \(47.35\%\), and \(47.34\%\)).  Our contributions are:
\begin{enumerate}
    \item We present a unified policy-loss analysis of token reweighting for SFT, showing that representative methods such as DFT, ProFit, and EAFT can be expressed as demonstration-action objectives with nonnegative reference-token coefficients.

    \item We identify a structural distinction between training-time token reweighting and post-hoc feature correction. Token weights can suppress or amplify supervised gradients, but they cannot reverse learned harmful features or extrapolate beneficial features within a fixed base-to-SFT coordinate system.

    \item We propose \Method{}, an entropy-guided adaptation-strength-control method that freezes both the pretrained model and the SFT delta, then learns bounded gates \(\lambda_{\ell,t}\in[\lambda_{\min},\lambda_{\max}]\) to perform suppression, reversal, and extrapolation using predictive entropy alone.

    \item Experiments across three model settings demonstrate consistent improvements over standard SFT and strong token-reweighting baselines, while ablations confirm the importance of extending the correction range beyond suppression.
\end{enumerate}

%% file: Sections/Related_Work.tex
\section{Related Work}
\label{app:rw-related}

\paragraph{Token-level reweighting for SFT.}
A growing line of work improves SFT by recognizing that demonstration tokens
should not contribute equally to training. Token Cleaning identifies
uninformative tokens through their estimated training influence
\citep{pang2025token}; DFT rescales token losses using the model probability
to improve the implicit reward structure of SFT
\citep{wu2026generalizationsftreinforcementlearning}; and ProFit masks
low-probability tokens to reduce overfitting to replaceable expressions
\citep{liu2026profitleveraginghighvaluesignals}. More recent methods incorporate
uncertainty or distributional preservation: EAFT uses predictive entropy to
identify confident conflicts \citep{diao2026entropyadaptivefinetuningresolvingconfident},
TALR adaptively reweights token losses to balance domain adaptation and general
capabilities \citep{lin2026sft}, and EKSFT combines entropy- and
KL-based token selection with additional regularization \citep{liu2026eksft}.
Beyond suppression and reweighting, Forgetting explicitly applies negative
updates to tokens judged harmful \citep{taheri2026forgettingnewmechanismbetter}.
These approaches intervene while the task update is being learned.
Our setting is complementary: \Method{} first completes ordinary SFT, freezes the
resulting task delta, and then learns how strongly that fixed delta should act
on each input and module. The distinction is therefore not simply positive
versus signed token weights, but training-time modification of the learning
trajectory versus post-hoc control of an already learned feature direction.

\paragraph{Post-training model editing and delta manipulation.}
A line of work studies how to manipulate fine-tuning updates after training.
Task Arithmetic represents fine-tuned models through their parameter
displacements from a pretrained model and composes multiple task vectors
through arithmetic operations \citep{ilharco2023editing}. TIES-Merging further
addresses interference among independently trained task vectors by resolving
redundant and conflicting updates \citep{yadav2023ties}. More recent methods
such as AdaMerging and WEMoE learn
how to combine multiple frozen task updates for multi-task transfer
\citep{yang2024adamergingadaptivemodelmerging,shen2024wemoe}. These methods
focus on selecting or composing different learned updates. In contrast,
\Method{} considers a different post-SFT control problem: given a single
completed SFT delta, how should its contribution be modulated for each
input and module? By learning token-conditioned gates around the original SFT
endpoint, \Method{} enables suppression, reversal, and extrapolation of the
same frozen delta rather than merging multiple experts.

%% file: Sections/Revisiting.tex
\section{Preliminaries}
\label{sec:rw-preliminaries}
Let $c_t=(x,y_{<t})$ be a demonstrated prefix, $y_t$ its next token, and
$\pi_{\theta,t}(a)=\operatorname{softmax}(z_\theta(c_t))_a$ the next-token policy.
We average over a fixed collection $\mathcal T$ of response positions, writing
$\avg{f_t}=|\mathcal T|^{-1}\sum_{t\in\mathcal T}f_t$.
All logarithms are natural. Standard supervised fine-tuning (SFT) minimizes
\begin{equation}
 \mathcal L_{\mathrm{SFT}}(\theta)=\avg{\ell_t(\theta)},
 \qquad \ell_t=-\log\pi_{\theta,t}(y_t),
 \qquad \nabla_{z_t}\ell_t=\pi_{\theta,t}-e_{y_t}.
 \label{eq:rw-sft}
\end{equation}
Our policy-loss view treats a prefix as a state and a vocabulary token as an
action. With an action distribution $\mu_t$ and action coefficient $A_t$, define
\begin{equation}
 \mathcal L_{\mathrm{policy}}(\theta)
 =-\avg{\sum_{a\in\mathcal V}
       \sg{\mu_t(a)A_t(a)}\log\pi_{\theta,t}(a)}.
 \label{eq:rw-policy}
\end{equation}
The coefficients are evaluated at the current iterate and held fixed during
its backward pass. Thus, this is a gradient surrogate, not necessarily
an equality of scalar objectives. SFT uses $\mu_t=\delta_{y_t}$ and
$A_t(y_t)=1$; token reweighting changes this coefficient. The equivalent
on-policy representation and its importance-weighting interpretation are given
in Appendix~\ref{app:rw-onpolicy}. Contexts remain teacher-forced: an on-policy
action expectation does not imply on-policy trajectories.

\subsection{Predictive entropy as a signed policy loss}
For $\pi_t=\pi_{\theta,t}$, let $H_t=-\sum_a\pi_t(a)\log\pi_t(a)$.
\begin{Lemma}[Entropy advantage]
\label{lem:rw-entropy}
For any differentiable parameterization,
\begin{equation}
 \nabla_\theta H_t
 =-\E_{a\sim\pi_t}\!\left[A_t^{\mathrm{ent}}(a)
             \nabla_\theta\log\pi_t(a)\right],
 \qquad A_t^{\mathrm{ent}}(a)=\log\pi_t(a)+H_t.
 \label{eq:rw-entropy-advantage}
\end{equation}
\end{Lemma}
Consequently, Eq.~\eqref{eq:rw-policy} with $\mu_t=\pi_t$ and
$A_t=A_t^{\mathrm{ent}}$ has exactly the entropy gradient. The advantage is
centered, $\E_{\pi_t}A_t^{\mathrm{ent}}=0$, and positive when
$\pi_t(a)>\exp(-H_t)$. Its direct logit gradient is
$\partial H_t/\partial z_t(a)=-\pi_t(a)A_t^{\mathrm{ent}}(a)$.
This identity supplies signed action-level feedback, but does not identify
correct actions. Our use of it will be restricted to scaling already
learned, frozen deltas, under an explicit feature-alignment assumption.

\section{Revisiting Token Correction Methods for SFT}
\label{sec:rw-correction}
\subsection{What nonnegative token weights control}
A demonstration-action corrector has a weighted CE component
\begin{equation}
 \mathcal L_w=-\avg{\sg{w_t}\log\pi_{\theta,t}(y_t)},
 \qquad w_t\geq0.
 \label{eq:rw-weighted}
\end{equation}
DFT and ProFit use reference probabilities, whereas EAFT uses predictive
entropy to weight supervision
\citep{wu2026generalizationsftreinforcementlearning, liu2026profitleveraginghighvaluesignals, diao2026entropyadaptivefinetuningresolvingconfident}.
Appendix~\ref{app:rw-catalogue} specifies their coefficients and distinguishes
weighted CE from additional loss terms.
\begin{proposition}[Local one-sidedness of weighted CE]
\label{prop:rw-onesided}
Holding $w_t$ fixed, a token's direct logit descent direction is
\begin{equation}
 -\nabla_{z_t}(w_t\ell_t)=w_t(e_{y_t}-\pi_t).
 \label{eq:rw-ray}
\end{equation}
It raises the demonstrated logit, lowers competing logits, or vanishes.
Nonnegative reweighting cannot reverse this particular CE direction.
\end{proposition}
This is a statement about one loss component, not an impossibility theorem
for shared-parameter training. 
The distinction we need is more specific:
$w_t$ controls learning an update; it does not directly control the
application of a completed update.

\subsection{A fixed delta defines suppression, reversal, and extrapolation}
After SFT, freeze both the pretrained map $W_m$ and the learned delta
$\Delta W_m$. A scalar gate at module $m$ and position $s$ gives
\begin{equation}
 u_{m,s}=W_mh_{m,s}+\lambda_{m,s}\Delta W_mh_{m,s}.
 \label{eq:rw-gated-delta}
\end{equation}
At the same module input, the delta is exactly
$\lambda_{m,s}\Delta W_mh_{m,s}$: $\lambda=0$ disables it,
$0<\lambda<1$ suppresses it, $\lambda=1$ preserves it,
$\lambda<0$ reverses it, and $\lambda>1$ extrapolates its strength.
This extends the allowable coefficients of a fixed delta, not the set of
models obtainable by every training-time reweighting method.
We use feature for this module-level delta contribution, not an
individually disentangled semantic unit. Its operator is frozen; hidden states
and final logit directions can still change with the gates.

\subsection{Why entropy can select fixed features}
Our premise concerns feature contributions: locally, task-supporting
deltas reduce uncertainty, while interfering deltas increase it.
This is not assumed for unrestricted feature learning. A binary margin model
makes the premise concrete. Let the correct-versus-incorrect margin be
$m(\lambda)=m_0+\sum_i\lambda_i a_i$, with fixed contributions $a_i$ and
correct probability $p=\operatorname{sigmoid}(m)$. Then
\begin{equation}
 \frac{\partial H}{\partial\lambda_i}
   =-m\,p(1-p)a_i
   =-m\,\frac{\partial p}{\partial\lambda_i}.
 \label{eq:rw-margin}
\end{equation}

\begin{wrapstuff}[width=0.7\linewidth, l]
  \centering
  \includegraphics[width=\linewidth]{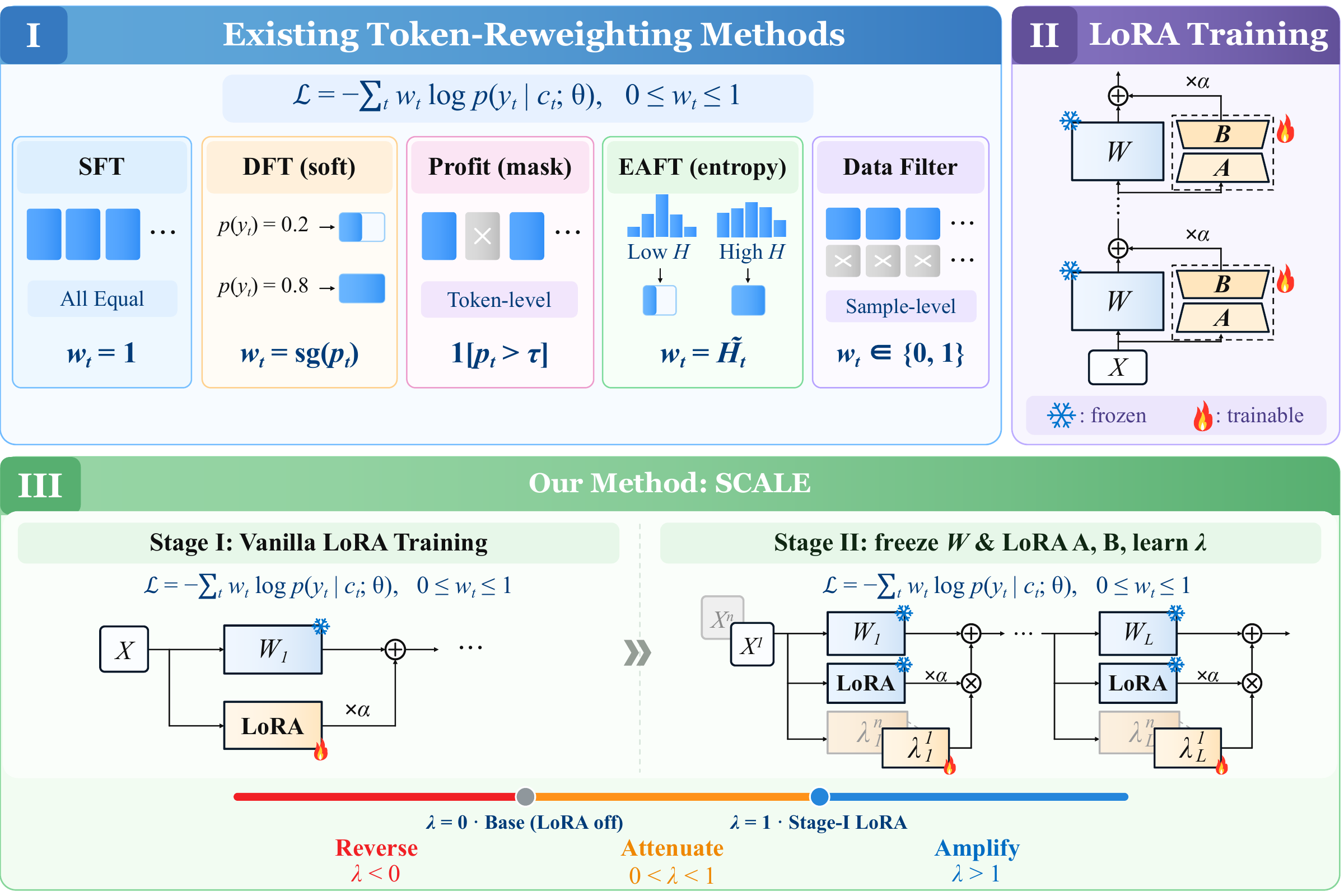}
  \captionof{figure}{Overview of baseline methods and SCALE.}
  \label{fig:method}
\end{wrapstuff}
Whenever $m>0$, entropy descent strengthens helpful contributions $a_i>0$
and weakens harmful ones $a_i<0$. Correctness defines helpfulness independently
of entropy; the positive margin supplies the needed semantic reference.
When $m<0$, the conclusion reverses. Appendix~\ref{app:rw-proofs} gives the
multiclass extension. This conditional mechanism is consistent with theoretical
work on avoiding spurious features through self-training
\citep{chen2020selftrainingavoidsusingspurious}, without importing its distributional guarantees
into a Transformer. We next express the required alignment in policy-loss
form for shared, input-conditioned gates.

%% file: Sections/Method.tex
\section{Method}
\label{sec:rw-method}
As shown in Figure~\ref{fig:method}, \Method{} separates delta acquisition from delta calibration: standard SFT
learns the delta, and predictive entropy subsequently selects its strength.


\subsection{Stage I: learning the task delta}
LoRA \citep{hu2022lora} parameterizes the delta operator as
\begin{equation}
 \Delta_{\phi,m}(h)=\Delta W_mh,
 \qquad \Delta W_m=\frac{\alpha}{r}B_mA_m.
 \label{eq:rw-lora}
\end{equation}
We train only its factors
$\phi=\{A_m,B_m\}_m$ with Eq.~\eqref{eq:rw-sft}, obtaining $\phi^\star$.
Stage II freezes the base, output head, and $\phi^\star$.
No delta direction is relearned during calibration.

\subsection{Stage II: identity-centered signed gates}
Each targeted module has an independent scalar gate generated by a linear
projection. Unlike continued SFT, this stage cannot modify the learned
delta; it only changes the coefficient applied to the frozen delta. For module $m$, let
$w_m\in\mathbb{R}^{d_m}$ and $b_m\in\mathbb{R}$ denote the gate weight and
bias:
\begin{equation}
 g_{m,s}=w_m^\top\sg{h_{m,s}}+b_m,
 \qquad
 v_{m,s}=\tanh(g_{m,s}/\tau),
 \qquad \tau>0.
 \label{eq:rw-gate-logit}
\end{equation}
We map this response to a bounded coefficient using
\begin{equation}
 \lambda_{m,s}
 =1+(1-L)\min(v_{m,s},0)+(U-1)\max(v_{m,s},0),
 \qquad L<0<1<U.
 \label{eq:rw-transform}
\end{equation}
Its image is $(L,U)$, and zero gate weights and biases recover the exact
Stage-I model. In the experiments, we use a deterministic
$10^{-6}$-scale perturbation and check initialization parity.
The bounded range limits module-level intervention magnitude; it is not an
output-divergence guarantee. Appendix~\ref{app:implementation} specifies the
implementation, including the nonsmooth join and range-ablation endpoints.

\subsection{Parameter efficiency}
\label{sec:rw-parameter-efficiency}
The proposed gating component introduces negligible extra parameters on top of
the frozen LoRA delta. For each targeted module $m$, the gate adds a linear
projection $(w_m,b_m)$, where $w_m\in\mathbb{R}^{d_m}$ and $b_m\in\mathbb{R}$,
resulting in only $d_m+1$ additional trainable parameters. Therefore, the
total number of Stage-II parameters is $\sum_{m\in\mathcal M}(d_m+1)$, where
$\mathcal M$ denotes the gated modules. In contrast, the Stage-I LoRA adapter
contains matrix-valued factors $A_m$ and $B_m$ with $r(d_m+d'_m)$ parameters per
module. Thus, the proposed calibration optimizes only a small fraction of the
parameters required to learn the original SFT delta, while keeping the base
model and LoRA adapter completely frozen during Stage II.

\subsection{Entropy-only training on response prefixes}
For $k>1$, let $S_{k,t}$ be the indices of the largest $k$ logits and define
\begin{equation}
 q_{\psi,t}(a)=\frac{\exp z_{\psi,t}(a)}
 {\sum_{b\in S_{k,t}}\exp z_{\psi,t}(b)},\quad a\in S_{k,t},
 \qquad
 F_k(\psi)=\frac{1}{\log k}\avg{H(q_{\psi,t})}.
 \label{eq:rw-topk}
\end{equation}
We use $k=100$ and response positions only. There is no CE, KL, or budget
term in Stage II, although the prefixes still come from demonstrations.
Taking $k=|\mathcal V|$ recovers normalized full-vocabulary entropy.
For $k<|\mathcal V|$, Eq.~\eqref{eq:rw-topk} is a conditional-entropy objective,
not the full entropy; all following top-$k$ derivatives are local to an
unchanged support (Appendix~\ref{app:rw-calculus}).

\subsection{Policy credit assignment within the frozen-delta family}
Write $q_t=q_{\psi,t}$. Lemma~\ref{lem:rw-entropy} gives
$A_{k,t}^{\mathrm{ent}}=\log q_{\psi,t}+H(q_{\psi,t})$.
For a gate coordinate $i$, let
$d_{t,i}=\partial z_{\psi,t}/\partial\lambda_i$ denote its local intervention
tangent. Then
\begin{equation}
 -\frac{\partial F_k}{\partial\lambda_i}
 =\frac{1}{\log k}\avg{\E_{q_t}\!\left[
 A_{k,t}^{\mathrm{ent}}(a)\frac{\partial\log q_t(a)}{\partial\lambda_i}
 \right]}
 =\frac{1}{\log k}\avg{\Cov_{q_t}(z_t,d_{t,i})}.
 \label{eq:rw-gate-credit}
\end{equation}
Equivalently, Eq.~\eqref{eq:rw-policy} uses $\mu_t=q_t$ and
$A_t=A_{k,t}^{\mathrm{ent}}/\log k$ (Appendix~\ref{app:rw-calculus}).
The average includes every causally affected response position, not only the
gate's own position. This is signed policy credit transmitted through delta
controls, rather than an unconstrained update to task features.

To characterize when entropy provides a useful correction signal, introduce fixed,
unobserved action values $r_t(a)$ and define
$V_k(\psi)=\avg{\E_{q_t}r_t(a)}$.
These values are used only for analysis. Let
$A_t^r=r_t-\E_{q_t}r_t$ denote the corresponding task advantage and define
$s_t(a)=\nabla_\psi\log q_t(a)$. For some $\beta>0$, the difference between
the entropy advantage and the task advantage in the gate parameter space can be
written as
\begin{equation}
 e_\psi=\avg{\E_{q_t}\!\left[
 (A_{k,t}^{\mathrm{ent}}-\beta A_t^r)s_t\right]},
 \qquad
 -\log k\,\nabla_\psi F_k
 =
 \beta\nabla_\psi V_k+e_\psi .
 \label{eq:rw-alignment-identity}
\end{equation}

\begin{assumption}[Gate-subspace entropy--task consistency]
\label{ass:rw-alignment}
At the operating point, the mismatch term satisfies
$\|e_\psi\|\leq\rho\beta\|\nabla_\psi V_k\|$
for some $0\leq\rho<1$.
\end{assumption}

The condition only requires consistency between entropy and task value along
the directions accessible through the gate parameters; it does not require
pointwise agreement between model confidence and correctness.
\begin{theorem}[Local task improvement by gate-only entropy descent]
\label{thm:rw-improvement}
Under Assumption~\ref{ass:rw-alignment}, differentiability, and
$\nabla_\psi V_k\neq0$, the complete-gradient update
$\psi^+=\psi-\eta\nabla_\psi F_k$ increases $V_k$ for sufficiently small
$\eta>0$. Specifically,
\begin{equation}
 V_k(\psi^+)-V_k(\psi)
 \geq
 \frac{\eta\beta(1-\rho)}{\log k}
 \|\nabla_\psi V_k\|^2
 +o(\eta).
 \label{eq:rw-value-improvement}
\end{equation}
\end{theorem}

Thus, entropy provides a valid local credit signal for scaling frozen deltas
when its directional preference is consistent with task value in the gate
subspace. The guarantee concerns local proxy-value improvement under this
condition, while the practical effectiveness of the method is evaluated
empirically.

\subsection{Why allow reversal and extrapolation after SFT?}
\begin{proposition}[Removing attenuation-boundary obstructions]
\label{prop:rw-boundary}
Let $F$ be differentiable near a local minimizer $\bar\lambda$ on $[0,1]^M$.
If $\nabla F(\bar\lambda)\neq0$, then, for $L<0<1<U$, an arbitrarily close
point in $(L,U)^M$ has strictly smaller $F$.
\end{proposition}
At $\bar\lambda_i=0$, a positive derivative calls for reversal; at
$\bar\lambda_i=1$, a negative derivative calls for extrapolation.
A decrease from $\lambda=1$ alone does not justify crossing zero.
The proposition concerns independent intervention coordinates; a shared gate
must also express the useful direction. It neither places the nonlinear
optimum at $L$ or $U$ nor guarantees that optimization finds it.

Freezing $\phi^\star$ fixes the delta operators and the meaning of these
coefficients throughout calibration. Interleaving CE would change the
operators being controlled, whereas post-hoc training isolates their terminal
use (Appendix~\ref{app:rw-posthoc}). This design complements task arithmetic
and entropy-based coefficient learning
\citep{ilharco2023editing,yang2024adamergingadaptivemodelmerging}; its distinctive intervention
is token- and module-conditioned signed scaling of one completed SFT delta.

%% file: Sections/Experiment.tex
\section{Experiments}
\label{sec:experiments}

In this section, we evaluate \Method{} across mathematical reasoning and code generation on multiple backbone models. Additional experiments, including ablation studies, realized gate behavior, and training-time analysis, are provided in Appendix~\ref{app:implementation}.

\subsection{Setup}
\label{sec:exp-setup}

\textbf{Experimental models and baselines.}
We conduct experiments on three widely used LLMs, including Qwen2.5-Math-1.5B~\citep{yang2024qwen2math}, Qwen2.5-Math-7B~\citep{yang2024qwen2math}, and Qwen3-4B-Base~\citep{yang2025qwen3}. Based on these backbones, we compare our proposed method \Method{} against representative adaptation baselines, including parameter-efficient fine-tuning methods such as LoRA~\citep{hu2022lora} and DoRA~\citep{liu2024dora}, as well as token-level optimization methods such as ProFit~\citep{liu2026profitleveraginghighvaluesignals}, DFT~\citep{wu2026generalizationsftreinforcementlearning}, EAFT~\citep{diao2026entropyadaptivefinetuningresolvingconfident}, and TALR~\citep{lin2026sft}.

\textbf{Datasets and evaluation.}
We organize our experiments by adaptation domains, including mathematical reasoning and code generation. For mathematical reasoning, we follow DFT~\citep{wu2026generalizationsftreinforcementlearning} in the choice of NuminaMath-CoT~\citep{numinamath2024} for fine-tuning and the Math500~\citep{hendrycks2021measuring}, Minerva~\citep{lewkowycz2022solving}, OlympiadBench~\citep{he2024olympiadbench}, AIME24, and AMC23 benchmarks. We further evaluate general capability retention after adaptation on MMLU-Pro~\citep{wang2024mmlu}, BBH~\citep{suzgun2023challenging}, and OpenBookQA~\citep{mihaylov2018can}. For code generation, we follow MoE$^2$-LoRA~\citep{yang2026moe2lora} in the choice of the Magicoder-OSS corpus~\citep{wei2023magicoder} for fine-tuning. We evaluate on HumanEval~\citep{chen2021evaluating}, HumanEval+~\citep{liu2023your}, and MBPP~\citep{austin2021program}.

\textbf{Implementation details.} We use two training stages throughout. In Stage I, a LoRA adapter with
rank $r=16$, scaling factor $\alpha=16$, and zero dropout is trained on all
linear layers for one epoch with total batch size 128, AdamW, peak learning
rate $5\times10^{-4}$ for math or $5\times10^{-5}$ for code, cosine decay, and 10\% warmup. In Stage II, the base
model and LoRA adapter are frozen; only the lightweight
$\tanh$-bounded gate is trained for a short period of time. To better understand the mechanism of the range of the gating component, we evaluate all 42 combinations
of lower bounds $L\in\{-6,\ldots,0\}$ and upper bounds
$U\in\{1,\ldots,6\}$, train a separate gate for each interval, and
\textbf{report the result of every interval in Figure~\ref{fig:heatmaps}}. The Ours rows in
Tables~\ref{tab:math_results} and \ref{tab:coding-results} additionally present the best intervals
in Figure~\ref{fig:heatmaps} on the corresponding evaluation benchmarks.

\subsection{Main results}
\label{sec:main-results}

\textbf{Mathematical reasoning and general-capability retention.}
Table~\ref{tab:math_results} reports the mathematical reasoning and general-capability retention results. On mathematical reasoning, \Method{} achieves state-of-the-art performance among all compared adaptation methods across all three backbones, obtaining the best results on every benchmark. This consistent improvement across different backbone families demonstrates the effectiveness of learning a bounded correction over the adaptation delta for enhancing reasoning ability. Meanwhile, \Method{} achieves strong general-capability retention with the highest retention Avg on both Qwen2.5-Math backbones and the second-highest on Qwen3-4B-Base. These results show that \Method{} combines improved mathematical reasoning with strong average general-capability retention.

\textbf{Code generation and downstream generalization.}
Table~\ref{tab:coding-results} reports the code-generation performance of \Method{}.
\Method{} consistently achieves the best overall performance across different backbone families, obtaining the highest Avg score on Qwen2.5-Math-1.5B, Qwen2.5-Math-7B, and Qwen3-4B-Base. These results extend the benefits of adaptation-delta control to code generation across the three backbones.

\input{tables/math}

\begin{figure}[ht]
\centering
\includegraphics[width=\linewidth]{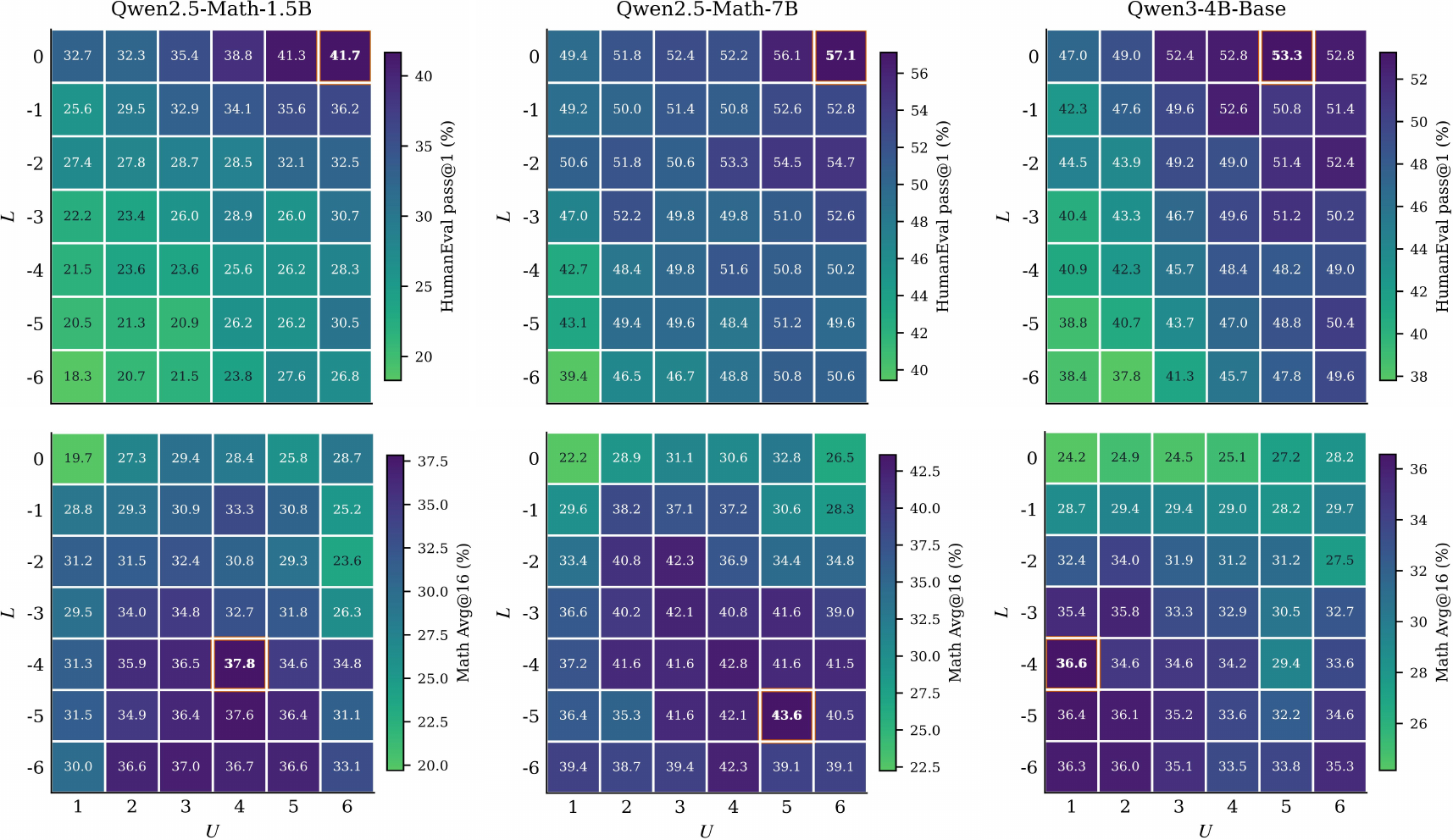}
\caption{Gate-range sensitivity across backbones. We evaluate all endpoint
pairs to characterize the intervention space. The selected maxima are reported
as oracle sweep results, while fixed-range results are used for controlled
comparisons.}
\label{fig:heatmaps}
\end{figure}
\vspace{-5pt}

\input{tables/codex}

\subsection{Discussion}
\label{sec:gate-range}

\textbf{Gate-range preferences across backbones.}
Figure~\ref{fig:heatmaps} shows a task-dependent pattern across all three backbones. High-scoring mathematical configurations occupy the negative-$L$ region. The two Qwen2.5-Math models also favor upper bounds above one, whereas Qwen3-4B-Base reaches its mathematical peak with a smaller upper bound. In the coding task, the high-scoring region lies near $L=0$ with larger positive upper bounds. The preferred balance of reversal and extrapolation therefore varies with both the task and the backbone.

\textbf{What the bounds contribute.}
At fixed $U=1$, negative lower bounds improve math on all three backbones; larger $U$ adds gains for both Qwen2.5-Math models. Expanding $[0,1]$ to $[0,6]$ improves code pass@1 on all three, while even $[0,1]$ exceeds LoRA in math. These patterns support reversal for math and extrapolation for code. Appendix~\ref{app:mechanism-analysis} reports realized gates and inference-time interventions.

\textbf{Interpreting the task difference.}
Math and code may use pretrained and newly adapted features differently. Qwen2.5-Math's specialization~\citep{yang2024qwen2math} may make some SFT features worth correcting, while code gains more from amplifying them. Qwen3-4B-Base~\citep{yang2025qwen3} shows the same split, suggesting that task--delta alignment also matters.

%% file: tables/math.tex
\begin{table*}[htbp!]
\caption{Mathematical reasoning and general-capability retention results (\%). Mathematical-reasoning scores are Avg@16. Boldface and underlining mark the best and second-best adapted scores in each column.}
\label{tab:math_results}
\centering
\scriptsize
\setlength{\tabcolsep}{2.8pt}
\renewcommand{\arraystretch}{0.92}
\newcommand{\best}[1]{{\textbf{#1}}}
\newcommand{\second}[1]{{\underline{#1}}}
\resizebox{\linewidth}{!}{
\begin{tabular}{@{}p{2.5cm}ccccc|c|ccc|c|}
\toprule
&\multicolumn{6}{c|}{\textbf{Mathematical reasoning}}&\multicolumn{4}{c}{\textbf{General Retention}}\\
\cmidrule(l){2-7}\cmidrule(r){8-11}
Method&Math500&Minerva&Olympiad&AIME24&AMC23&Avg&MMLU-P&BBH&OBQA&Avg\\
\midrule
\multicolumn{11}{c}{\textbf{Qwen2.5-Math-1.5B~\citep{yang2024qwen2math}}}\\
\midrule
\textcolor{gray}{Base}&\textcolor{gray}{21.87}&\textcolor{gray}{5.24}&\textcolor{gray}{9.18}&\textcolor{gray}{1.66}&\textcolor{gray}{12.81}&\textcolor{gray}{10.15}&\textcolor{gray}{20.24}&\textcolor{gray}{26.32}&\textcolor{gray}{30.40}&\textcolor{gray}{25.65}\\
LoRA~\citep{hu2022lora}&39.60&9.49&10.80&0.41&17.03&15.47&27.89&31.76&33.80&31.15\\
DoRA~\citep{liu2024dora}&40.10&9.39&10.69&1.24&16.88&15.66&27.58&32.20&33.60&31.13\\
ProFit~\citep{liu2026profitleveraginghighvaluesignals}&56.12&\second{21.93}&22.69&2.29&31.56&26.92&28.02&29.78&33.80&30.53\\
DFT~\citep{wu2026generalizationsftreinforcementlearning}&\second{66.49}&21.13&\second{27.61}&\second{8.34}&\second{40.47}&\second{32.81}&28.76&32.92&33.80&31.83\\
EAFT~\citep{diao2026entropyadaptivefinetuningresolvingconfident}&38.79&8.56&9.44&0.21&15.00&14.40&\second{28.78}&\second{32.93}&\best{34.80}&\second{32.17}\\
TALR~\citep{lin2026sft}&60.28&18.54&23.29&3.32&35.31&28.15&\best{28.83}&32.06&33.80&31.56\\
\rowcolor{gray!15}
Ours&\best{71.60}&\best{31.44}&\best{31.34}&\best{8.55}&\best{46.25}&\best{37.84}&28.55&\best{38.67}&32.80&\best{33.34}\\
\midrule
\multicolumn{11}{c}{\textbf{Qwen2.5-Math-7B~\citep{yang2024qwen2math}}}\\
\midrule
\textcolor{gray}{Base}&\textcolor{gray}{20.57}&\textcolor{gray}{7.89}&\textcolor{gray}{6.49}&\textcolor{gray}{3.33}&\textcolor{gray}{11.88}&\textcolor{gray}{10.03}&\textcolor{gray}{37.86}&\textcolor{gray}{45.75}&\textcolor{gray}{34.40}&\textcolor{gray}{39.34}\\
LoRA~\citep{hu2022lora}&49.94&13.91&16.43&3.75&25.62&21.93&38.55&49.12&35.40&41.02\\
DoRA~\citep{liu2024dora}&49.90&12.20&15.86&3.12&25.00&21.22&38.69&50.35&35.40&41.48\\
ProFit~\citep{liu2026profitleveraginghighvaluesignals}&65.71&24.14&31.32&7.93&43.12&34.44&39.47&49.85&36.00&\second{41.77}\\
DFT~\citep{wu2026generalizationsftreinforcementlearning}&\second{68.95}&\underline{27.75}&\second{33.67}&\second{8.96}&\second{48.28}&\second{37.52}&39.84&50.08&33.20&41.04\\
EAFT~\citep{diao2026entropyadaptivefinetuningresolvingconfident}&50.41&14.10&18.67&4.38&26.09&22.73&38.71&50.12&36.20&41.68\\
TALR~\citep{lin2026sft}&68.19&23.92&33.16&7.51&46.09&35.77&\second{40.25}&\second{51.15}&33.80&41.73\\
\rowcolor{gray!15}
Ours&\best{76.97}&\best{35.51}&\best{36.12}&\best{12.51}&\best{56.88}&\best{43.60}&\best{40.87}&\best{64.52}&30.80&\best{45.40}\\
\midrule
\multicolumn{11}{c}{\textbf{Qwen3-4B-Base~\citep{yang2025qwen3}}}\\
\midrule
\textcolor{gray}{Base}&\textcolor{gray}{46.36}&\textcolor{gray}{12.67}&\textcolor{gray}{19.29}&\textcolor{gray}{6.87}&\textcolor{gray}{24.84}&\textcolor{gray}{22.01}&\textcolor{gray}{46.39}&\textcolor{gray}{47.47}&\textcolor{gray}{37.80}&\textcolor{gray}{43.89}\\
LoRA~\citep{hu2022lora}&47.30&11.54&15.90&3.54&21.72&20.00&51.54&55.54&\second{37.40}&48.16\\
DoRA~\citep{liu2024dora}&46.53&11.29&16.07&1.86&22.34&19.62&51.39&55.37&36.80&47.85\\
ProFit~\citep{liu2026profitleveraginghighvaluesignals}&65.75&21.53&30.16&5.42&40.47&32.67&52.04&50.24&37.00&46.43\\
DFT~\citep{wu2026generalizationsftreinforcementlearning}&\second{68.26}&\second{25.00}&\second{32.27}&7.92&39.84&\second{34.66}&\second{52.08}&49.27&35.40&45.58\\
EAFT~\citep{diao2026entropyadaptivefinetuningresolvingconfident}&46.92&12.80&15.68&1.45&21.88&19.75&52.04&\second{55.56}&\best{37.60}&\best{48.40}\\
TALR~\citep{lin2026sft}&66.62&20.18&32.11&\second{9.59}&\second{40.94}&33.89&51.78&53.11&33.60&46.16\\
\rowcolor{gray!15}
Ours&\best{69.85}&\best{25.41}&\best{32.77}&\best{9.99}&\best{44.84}&\best{36.57}&\best{52.25}&\best{55.97}&36.80&\second{48.34}\\
\bottomrule
\end{tabular}
}
\end{table*}

%% file: tables/codex.tex
\begin{table}[htbp!]
\caption{Code-generation results (\%). Boldface and underlining mark the best and second-best adapted scores in each column.}
\label{tab:coding-results}
\centering
\scriptsize
\setlength{\tabcolsep}{2.5pt}
\renewcommand{\arraystretch}{0.92}
\resizebox{\linewidth}{!}{
\begin{tabular}{@{}p{2.5cm}ccc|c|ccc|c|ccc|c|}
\toprule
&\multicolumn{4}{c|}{\textbf{Qwen2.5-Math-1.5B}}&\multicolumn{4}{c|}{\textbf{Qwen2.5-Math-7B}}&\multicolumn{4}{c}{\textbf{Qwen3-4B-Base}}\\
\cmidrule(lr){2-5}\cmidrule(lr){6-9}\cmidrule(lr){10-13}
Method&HE&HE+&MBPP&Avg&HE&HE+&MBPP&Avg&HE&HE+&MBPP&Avg\\
\midrule
\textcolor{gray}{Base}&\textcolor{gray}{28.66}&\textcolor{gray}{17.68}&\textcolor{gray}{32.80}&\textcolor{gray}{26.38}&\textcolor{gray}{42.28}&\textcolor{gray}{39.02}&\textcolor{gray}{34.20}&\textcolor{gray}{38.50}&\textcolor{gray}{43.09}&\textcolor{gray}{39.02}&\textcolor{gray}{37.00}&\textcolor{gray}{39.70}\\
LoRA~\citep{hu2022lora}&34.15&25.61&30.80&30.19&48.37&43.29&\textbf{38.00}&43.22&\underline{51.22}&\underline{45.73}&\underline{40.20}&\underline{45.72}\\
DoRA~\citep{liu2024dora}&31.91&28.05&31.20&30.39&46.14&43.29&36.20&41.88&50.81&45.12&40.00&45.31\\
DFT~\citep{wu2026generalizationsftreinforcementlearning}&34.96&29.88&33.60&32.81&48.58&42.07&36.00&42.22&44.51&38.41&37.20&40.04\\
ProFit~\citep{liu2026profitleveraginghighvaluesignals}&35.57&30.49&\underline{33.80}&33.29&47.56&43.29&35.60&42.15&48.78&43.90&40.00&44.23\\
EAFT~\citep{diao2026entropyadaptivefinetuningresolvingconfident}&31.30&29.27&32.80&31.12&46.75&41.46&36.40&41.54&45.73&40.85&39.80&42.13\\
TALR~\citep{lin2026sft}&\underline{36.59}&\textbf{33.54}&\textbf{34.00}&\underline{34.71}&\underline{49.39}&\underline{44.51}&37.00&\underline{43.63}&44.72&39.63&37.80&40.72\\
\rowcolor{gray!15}
Ours&\textbf{41.67}&\underline{31.10}&31.40&\textbf{34.72}&\textbf{57.11}&\textbf{46.95}&\textbf{38.00}&\textbf{47.35}&\textbf{53.25}&\textbf{47.56}&\textbf{41.20}&\textbf{47.34}\\
\bottomrule
\end{tabular}
}
\end{table}

%% file: Sections/Appendix.tex
\providecommand{\sg}[1]{\operatorname{sg}\!\left[#1\right]}
\providecommand{\E}{\mathbb{E}}
\providecommand{\Cov}{\operatorname{Cov}}
\providecommand{\avg}[1]{\frac{1}{|\mathcal T|}\sum_{t\in\mathcal T} #1}

\newpage
\appendix


\section{Notation}
We list the notation used in this paper in Table~\ref{tab:notation}.
\begin{table}[h]
\caption{Notation used throughout the paper.}
\label{tab:notation}
\centering
\small
\setlength{\tabcolsep}{5pt}
\begin{tabular}{ll}
\toprule
Symbol & Definition \\
\midrule
$x$ & Input prompt \\
$y=(y_1,\dots,y_T)$ & Demonstrated response tokens \\
$c_t=(x,y_{<t})$ & Prefix context at token position $t$ \\
$a\in\mathcal V$ & Next-token action from vocabulary $\mathcal V$ \\
$\pi_\theta(a\mid c_t)$ & Model next-token distribution parameterized by $\theta$ \\
$z_t$ & Output logits at position $t$ \\
$H_t$ & Predictive entropy of $\pi_\theta(\cdot\mid c_t)$ \\
$A_t^{\mathrm{ent}}(a)$ & Entropy advantage:
$\log\pi_\theta(a\mid c_t)+H_t$ \\
\midrule
$\theta_0$ & Pretrained model parameters \\
$\phi^\star$ & Frozen LoRA parameters after Stage-I SFT \\
$\Delta\phi^\star$ & Learned SFT delta) \\
$W_\ell$ & Frozen pretrained weight of layer $\ell$ \\
$\Delta W_\ell$ & LoRA delta added to layer $\ell$ \\
$h_{\ell,t}$ & Hidden representation at layer $\ell$, position $t$ \\
$\lambda_{\ell,t}$ & Gate coefficient applied to the frozen delta \\
$\psi$ & Trainable gate parameters \\
$g_{\psi,\ell,t}$ & Gate logit before transformation \\
\midrule
$w_t$ & Token weight used by training-time reweighting methods \\
$\mathcal L_{\mathrm{SFT}}$ & Standard supervised fine-tuning loss \\
$\mathcal L_w$ & Weighted token-level SFT objective \\
$F_k(\psi)$ & Top-$k$ entropy objective optimized by the gate \\
$q_t$ & Renormalized top-$k$ predictive distribution \\
$V_k(\psi)$ & Analytical task-value proxy under gate parameters \\
$e_\psi$ & Entropy--task advantage mismatch in gate space \\
\midrule
$L,U$ & Lower and upper bounds of the signed gate range \\
$\lambda<0$ & delta reversal regime \\
$0\leq\lambda<1$ & delta suppression regime \\
$\lambda=1$ & Original SFT delta strength \\
$\lambda>1$ & delta extrapolation regime \\
\bottomrule
\end{tabular}
\end{table}

\section{Policy-loss Identities and Their Assumptions}
\label{app:rw-onpolicy}
\subsection{An exact gradient identity for SFT}
Let $q_t^{\mathrm{demo}}(a)=\mathbf 1[a=y_t]$ and assume a finite vocabulary
with strictly positive softmax probabilities. Since demonstrated contexts
are held fixed,
\begin{align}
 \nabla_\theta\mathcal L_w
 &=-\avg{\sum_a q_t^{\mathrm{demo}}(a)w_t
                     \nabla_\theta\log\pi_t(a)} \notag\\
 &=-\avg{\E_{a\sim\pi_t}\!\left[
       \frac{w_t\mathbf 1[a=y_t]}{\pi_t(a)}
       \nabla_\theta\log\pi_t(a)\right]}.
 \label{eq:rw-is-sft}
\end{align}
Equation~\eqref{eq:rw-is-sft} is the same expected gradient, not the derivative
of an expected reward obtained by differentiating the importance ratio.
The coefficient is stopped in the surrogate. The factor $1/\pi_t(y_t)$ is
cancelled in expectation by the sampling probability; it does not make the
deterministic CE logit gradient unbounded. This representation is related to
the RL analysis motivating DFT \citep{wu2026generalizationsftreinforcementlearning}.

An action-independent baseline can be subtracted because
$\E_{\pi_t}\nabla\log\pi_t=0$. For example, subtracting $w_t$ gives a signed
coefficient on the full vocabulary without changing the weighted CE gradient.
Therefore, signed coefficients alone are not an invariant distinction between
SFT and entropy minimization. The relevant objects are the induced gradients
and their trainable coordinates: one is a reweighted demonstration CE
direction; the other is entropy feedback applied to frozen-delta controls.
Neither representation makes the demonstrated prefix distribution on-policy.

\subsection{Reference-token entropy threshold}
For a distribution $p$ and reference token $y$, consider the
\emph{probability-space} path $p_\epsilon=(1-\epsilon)p+\epsilon e_y$.
Direct differentiation gives
\begin{equation}
 \left.\frac{\mathrm d}{\mathrm d\epsilon}H(p_\epsilon)\right|_0
 =-\log p_y-H(p)=-A^{\mathrm{ent}}(y).
 \label{eq:rw-mixture-path}
\end{equation}
Thus $p_y>\exp[-H(p)]$ identifies entropy decrease along this particular
mixture path. This threshold connects the reference token to the signed
entropy advantage; it is not an optimal weighting rule for ordinary SFT.
Indeed, a CE step in logits follows $z(\eta)=z+\eta(e_y-p)$ and induces
\begin{equation}
 \left.\frac{\mathrm dp}{\mathrm d\eta}\right|_0
   =C(p)(e_y-p),
 \qquad C(p)=\operatorname{diag}(p)-pp^\top,
 \label{eq:rw-ce-probability-path}
\end{equation}
which generally differs from $e_y-p$. A parameter update additionally depends
on the network Jacobian. Our method consequently uses the actual gate
sensitivity of entropy rather than treating the reference threshold as a
sufficient test for CE or feature correction.

\section{Token Correction Objectives}
\label{app:rw-catalogue}
Table~\ref{tab:rw-correctors} lists coefficients of the weighted CE component.
Here $p_t=\pi_t(y_t)$, $\ell_t=-\log p_t$, and $m_t$ is a preprocessing mask.
A method belongs to Eq.~\eqref{eq:rw-weighted} as a gradient objective only
when its coefficient is held fixed during differentiation.
\begin{table}[t]
\centering
\small
\setlength{\tabcolsep}{5pt}
\renewcommand{\tabularxcolumn}[1]{m{#1}}
\newcommand{\correctorcell}[1]{\parbox[c][2.8\baselineskip][c]{\linewidth}{\raggedright #1}}
\begin{tabularx}{\linewidth}{@{}>{\raggedright\arraybackslash}m{0.20\linewidth} >{\raggedright\arraybackslash}m{0.29\linewidth} >{\raggedright\arraybackslash}X@{}}
\toprule
Method & CE coefficient $w_t$ & Signal / qualification\\
\midrule
\correctorcell{SFT} & \correctorcell{$1$} & \correctorcell{Uniform demonstration weighting.}\\
\correctorcell{DFT} & \correctorcell{$\sg{p_t}$} & \correctorcell{Probability-weighted CE \citep{wu2026generalizationsftreinforcementlearning}.}\\
\correctorcell{ProFit} & \correctorcell{$\mathbf 1[\sg{p_t}>\tau_p]$} & \correctorcell{Probability mask \citep{liu2026profitleveraginghighvaluesignals}.}\\
\correctorcell{EAFT} & \correctorcell{$H_t^{\mathrm{top}\text{-}20}/\log20$} & \correctorcell{Normalized top-$20$ entropy; frozen-coefficient CE interpretation \citep{diao2026entropyadaptivefinetuningresolvingconfident}.}\\
\correctorcell{TALR} & \correctorcell{$\max\{w_{\min},\sg{e^{-\ell_t/\tau_\ell}}\}$} & \correctorcell{Loss-dependent weight, dynamic temperature, and floor \citep{lin2026sft}.}\\
\correctorcell{Offline token mask} & \correctorcell{$m_t\in\{0,1\}$} & \correctorcell{Fixed retained positions, e.g., token cleaning \citep{pang2025token}.}\\
\bottomrule
\end{tabularx}
\caption{Nonnegative coefficients of demonstration-action CE components.
Entries describe CE components. The EAFT row reports its published entropy
weight without specifying a stop-gradient convention.}
\label{tab:rw-correctors}
\end{table}

\paragraph{Differentiated weights.}
If a coefficient $w_t(\theta)$ is differentiated, then
\begin{equation}
 \nabla_\theta\bigl(w_t(\theta)\ell_t(\theta)\bigr)
 =w_t\nabla_\theta\ell_t+\ell_t\nabla_\theta w_t.
 \label{eq:rw-weight-derivative}
\end{equation}
Proposition~\ref{prop:rw-onesided} applies to the first term only. The
published EAFT loss uses normalized predictive entropy as its weighting
statistic, rather than an explicit two-argument function of reference
probability and entropy. Reference probability still enters the CE factor.
Its implementation's gradient convention must be specified independently.
For TALR, the practical rule uses a weight floor and a temperature based on
batch loss statistics; the nonnegative-ray conclusion is unchanged.

\paragraph{Objectives outside this class.}
EKSFT includes masked CE together with entropy and KL terms
\citep{liu2026eksft}. Its CE component can be entered in
Eq.~(\ref{eq:rw-weighted}), but its entire objective cannot. Likewise,
\citet{taheri2026forgettingnewmechanismbetter} explicitly subtract a loss on negative tokens.
That method is not constrained by nonnegative token weights. The
one-sidedness result therefore applies to nonnegative weighted CE components.

\section{Empirical Motivation: Delta Entropy and Local Feature Selection}
\label{app:rw-diagnostic}

\label{app:entropy-diagnostic}

Before training the gate, we ask whether the entropy effect of a \emph{completed}
LoRA delta carries directional information about how that delta should be
used. For matched teacher-forced prefixes, define
\begin{equation}
    \Delta H_t
    =
    H(\pi_{\mathrm{LoRA},t})-H(\pi_{\mathrm{Base},t}).
    \label{eq:rw-delta-entropy}
\end{equation}
The diagnostic suppresses delta contributions at positions selected by
the sign of $\Delta H_t$, with the base model and LoRA adapter frozen. This
is a hard forward intervention without gate training.
Table~\ref{tab:entropy-statistic} compares sign-based suppression with LoRA and Ours across five math benchmarks.

\begin{table}[t]
\caption{Zero-threshold entropy-sign ablation on \Qwen{}-1.5B. Negative and Positive suppress deltas where $\Delta H<0$ and $\Delta H>0$, respectively. Scores are Avg@16. Boldface and underlining mark the best and second-best scores in each column.}
\label{tab:entropy-statistic}
\centering
\small
\setlength{\tabcolsep}{3.5pt}
\renewcommand{\arraystretch}{1.1}
\begin{tabular*}{\linewidth}{@{\extracolsep{\fill}}lrrrrrr@{}}
\toprule
Method & Math500 & Minerva & Olympiad & AIME24 & AMC23 & $\Avg$ \\
\midrule
Negative & 30.85 & 6.54 & 8.85 & 1.24 & 15.47 & 12.59 \\
LoRA & 39.60 & 9.49 & 10.80 & 0.41 & 17.03 & 15.47 \\
Positive & \underline{47.99} & \underline{15.24} & \underline{20.08} & \underline{4.16} & \underline{25.47} & \underline{22.59} \\
Ours & \textbf{71.60} & \textbf{31.44} & \textbf{31.34} & \textbf{8.55} & \textbf{46.25} & \textbf{37.84} \\
\bottomrule
\end{tabular*}
\end{table}

Suppressing entropy-decreasing delta contributions yields 12.59 Math
Avg@16, compared with 15.47 for LoRA, while suppressing
entropy-increasing contributions yields 22.59. This ordering motivates using entropy change
to select which frozen delta contributions to weaken and learning a
continuous correction gate. The diagnostic measures finite differences
between frozen models; the local alignment conditions for task-improving
entropy descent are developed separately in Appendix~\ref{app:rw-proofs}.

\section{Causal Gate Calculus, Top-$k$ Entropy, and Stopped Gradients}
\label{app:rw-calculus}
\subsection{Intervention coordinates versus shared gate parameters}
Collect all module--position gate values for the fixed examples into
$\boldsymbol\lambda$. Let $G_t(\boldsymbol\lambda;\theta_0,\phi^\star)$ be the
logits when these values are supplied as external controls. The learned model
satisfies $z_{\psi,t}=G_t(\Lambda(\psi))$, where $\Lambda$ includes the dependence
of gates on their current hidden inputs. At a differentiable operating point,
\begin{equation}
 z_t(\boldsymbol\lambda+\delta\lambda)
 =z_t(\boldsymbol\lambda)+J_t^\lambda\delta\lambda
     +o(\|\delta\lambda\|),\qquad
 J_t^\psi=J_t^\lambda B,\quad
 B=\frac{\partial\Lambda}{\partial\psi}.
 \label{eq:rw-causal-jacobian}
\end{equation}
A column $d_{t,i}$ of $J_t^\lambda$ is zero when gate $i$ is outside the
causal computation of prediction $t$. Otherwise it may affect that prediction
even if located at an earlier sequence position. Consequently,
\begin{equation}
 \frac{\partial F_k}{\partial\lambda_i}
 =\frac{1}{\log k}\avg{\left\langle
       \nabla_{z_t}H(q_t),d_{t,i}\right\rangle}.
 \label{eq:rw-all-position-credit}
\end{equation}
No globally affine base-to-SFT logit model is assumed. Frozen $\Delta W_m$
fixes an operator, not $\Delta W_mh_{m,s}$ across inputs or training steps.
A scalar gate rescales the entire module delta at a position; it cannot
independently separate all semantic components within that delta.
Likewise, $B^\top$ is a chain-rule pullback, not generally an orthogonal
projection. An advantageous independent intervention need not be expressible
by shared gate parameters.

\subsection{Exact policy loss for the implemented conditional entropy}
Fix a neighborhood with unchanged top-$k$ support $S_t$. Extend $q_t$ by zero
outside $S_t$, and let $\mathcal H_t=H(q_t)$.
Then
\begin{equation}
 \frac{\partial\mathcal H_t}{\partial z_t(a)}
 =\begin{cases}
   -q_t(a)\bigl(\log q_t(a)+\mathcal H_t\bigr),&a\in S_t,\\
   0,&a\notin S_t.
  \end{cases}
 \label{eq:rw-topk-logit-gradient}
\end{equation}
Since $\E_{q_t}A_{k,t}^{\mathrm{ent}}=0$ and
$\log q_t(a)=\log\pi_t(a)-\log\sum_{b\in S_t}\pi_t(b)$,
\begin{equation}
 \nabla_\psi\mathcal H_t
 =-\E_{q_t}[A_{k,t}^{\mathrm{ent}}\nabla_\psi\log q_t]
 =-\E_{q_t}[A_{k,t}^{\mathrm{ent}}\nabla_\psi\log\pi_t].
 \label{eq:rw-topk-policy}
\end{equation}
Thus the implemented loss has the policy surrogate in
Eq.~\eqref{eq:rw-policy} with $\mu_t=q_t$ and
$A_t=A_{k,t}^{\mathrm{ent}}/\log k$. It uses all actions in the retained
support, not the full vocabulary when $k<|\mathcal V|$.
At support ties, the smooth proof is replaced by a branchwise analysis; the
reported procedure recomputes top-$k$ every forward pass.

For support mass $m_t=\sum_{a\in S_t}\pi_t(a)$ and normalized tail
$\widetilde\pi_t$, full entropy decomposes as
\begin{equation}
 H(\pi_t)=h_{\mathrm b}(m_t)+m_tH(q_t)
               +(1-m_t)H(\widetilde\pi_t).
 \label{eq:rw-entropy-tail}
\end{equation}
Decreasing $H(q_t)$ does not by itself decrease these other terms. Similarly,
$V_k$ is a conditional action value. It is not the full-vocabulary expected
value unless the corresponding tail effects are controlled.

\subsection{What detaching a gate input changes}
For one module $u=Wh+\lambda(h)\Delta Wh$, the complete hidden-state Jacobian is
\begin{equation}
 \frac{\partial u}{\partial h}
 =W+\lambda\Delta W+(\Delta Wh)(\nabla_h\lambda)^\top.
 \label{eq:rw-full-module-jacobian}
\end{equation}
Using $\lambda=\lambda(\sg{h})$ drops the last term from the backward graph.
Hidden states still depend on upstream gates, so freezing $W$ and $\Delta W$
does not remove this difference. Write $\mathrm D_\psi F_k$ for the resulting
semi-gradient and $\nabla_\psi F_k$ for the ordinary derivative of the numerical
forward function. Entropy and policy surrogates remain gradient-equivalent
\emph{within the same stopped graph}, but ordinary descent guarantees require
an additional check.

Define scaled complete and implemented directions by
\begin{equation}
 G_k=-\log k\,\nabla_\psi F_k,
 \qquad \widehat G_k=-\log k\,\mathrm D_\psi F_k,
 \qquad b_{\mathrm{sg}}=\widehat G_k-G_k.
 \label{eq:rw-stopped-directions}
\end{equation}
For the conditional task value, Eq.~\eqref{eq:rw-alignment-identity} becomes
\begin{equation}
 \widehat G_k=\beta\nabla_\psi V_k+e_\psi+b_{\mathrm{sg}}.
 \label{eq:rw-stopped-alignment}
\end{equation}
Hence Theorem~\ref{thm:rw-improvement} also holds for the semi-gradient step if
$\|e_\psi+b_{\mathrm{sg}}\|\leq\rho\beta\|\nabla_\psi V_k\|$, with $\rho<1$.
This is an additional condition, not a consequence of freezing.

For the full value $V=\avg{\E_{\pi_t}r_t}$, define
$b_{\mathrm{tail}}=\beta(\nabla_\psi V_k-\nabla_\psi V)$. Then
\begin{equation}
 \widehat G_k
 =\beta\nabla_\psi V+
      \underbrace{e_\psi+b_{\mathrm{sg}}+b_{\mathrm{tail}}}_{e_{\mathrm{total}}}.
 \label{eq:rw-total-error}
\end{equation}
The sufficient condition becomes
$\|e_{\mathrm{total}}\|\leq\rho\beta\|\nabla_\psi V\|$.
These terms separate feature-selection mismatch, stopped-gradient error, and
support truncation. The experiments do not directly measure them.
Complete-gradient gate training sets $b_{\mathrm{sg}}=0$ but is a distinct backward-pass variant; it is not
substituted for the reported method here.

\section{Post-hoc Calibration and a Fixed delta Reference}
\label{app:rw-posthoc}
Write the entropy objective as $F_\phi(\psi)$ to expose its dependence on the
LoRA factors. \Method{} first obtains $\phi^\star$ and then optimizes the fixed
forward objective $F_{\phi^\star}(\psi)$. The delta operators, not necessarily
the logit tangents, remain fixed. In particular, setting every gate to one
continues to identify the same Stage-I model, and setting every gate to zero
recovers the base when all other parameters are unchanged.

For interleaved iterates $(\phi^{(j)},\psi^{(j)})$, adding and subtracting
$F_{\phi^{(j)}}(\psi^{(j+1)})$ yields
\begin{align}
 &F_{\phi^{(j+1)}}(\psi^{(j+1)})-F_{\phi^{(j)}}(\psi^{(j)})\notag\\
 &=\underbrace{F_{\phi^{(j)}}(\psi^{(j+1)})
                   -F_{\phi^{(j)}}(\psi^{(j)})}_{\Delta_{\mathrm{gate}}^{(j)}}
  +\underbrace{F_{\phi^{(j+1)}}(\psi^{(j+1)})
                   -F_{\phi^{(j)}}(\psi^{(j+1)})}_{\Delta_{\mathrm{delta}}^{(j)}}.
 \label{eq:rw-drift}
\end{align}
The second term has no fixed sign. Freezing removes it and isolates the
question of how to use a completed delta. This identity does not establish
that post-hoc optimization outperforms joint or alternating optimization:
a joint CE--entropy objective can also be a fixed function of its combined
parameters.

Even for frozen $\phi^\star$, monotone loss descent requires a genuine
descent direction and an appropriate step size. The deterministic,
complete-gradient statements in Appendix~\ref{app:rw-proofs} are local to
smooth branches; minibatch noise, optimizer state, top-$k$ changes, and
stopped gate inputs require their own conditions. The post-hoc design fixes
the intervention reference, not these optimization issues.

\section{More Experiments}
\label{app:implementation}

\subsection{More Implementation Details}

\paragraph{Stage-I training.}
For math, the first-stage LoRA is trained for one epoch on 100k NuminaMath-CoT examples, following the data choice of DFT~\citep{wu2026generalizationsftreinforcementlearning}. For code, a separate first-stage LoRA is trained for one epoch on 75{,}197 Magicoder-OSS examples, following the data choice of MoE$^2$-LoRA~\citep{yang2026moe2lora}. Both use total batch size 128, rank 16, alpha 16, all linear modules as targets, and dropout 0; the learning rates are $5\times10^{-4}$ for math and $5\times10^{-5}$ for code. Stage II treats the resulting LoRA operators as the fixed delta family to be calibrated; it does not relearn their factors.

\paragraph{Stage-II gate training.}
For math, the gate is trained on the first 20k NuminaMath-CoT examples; for code, it is trained on a 20k-example subset of Magicoder-OSS. Each gate run uses 157 optimizer steps with total batch size 128, learning rate $5\cdot10^{-4}$, cosine schedule, zero warmup, and a cutoff length of 4096 tokens. Each targeted module owns a linear map
\[
g_{\psi,m,t}=w_m^\top\operatorname{sg}(x_{m,t})+b_m
\]
from its detached current input to a scalar logit, and these gate parameters are the only Stage-II trainable components. The reported runs use a fixed deterministic FP32 tiny initialization for the gate weights and zero bias. Step-zero parity with the frozen Stage-I model is verified before optimization. The normalized top-$k$ entropy uses $k=100$ and response-token predictions only.

Detaching gate inputs removes the derivative through their dependence on
upstream gates. Stage II therefore uses the semi-gradient analyzed in
Appendix~\ref{app:rw-calculus}. Likewise, the top-$k$ objective is the conditional entropy of the retained support, not the full-vocabulary entropy. The policy-loss identity remains exact on a locally
fixed support within the same stopped graph; transferring a complete-gradient task-improvement statement to the implemented update requires the additional error condition in Eq.~\eqref{eq:rw-stopped-alignment}.

\paragraph{Evaluation details.}
Mathematical-reasoning and general-capability retention evaluations use
task-native zero-shot prompts without a system message and EOT-aligned
assistant termination.  All methods are evaluated with vLLM inference under the same decoding configuration.


\subsection{Ablation studies}
\label{app:ablations}

\begin{table}[t]
\caption{Calibration-loss ablation on \Qwen{}-1.5B (\%). Scores are Avg@16; boldface and underlining mark the best and second-best values.}
\label{tab:gate-objective-ablation}
\centering
\small
\setlength{\tabcolsep}{3.5pt}
\renewcommand{\arraystretch}{1.1}
\begin{tabular*}{\linewidth}{@{\extracolsep{\fill}}lrrrrrr@{}}
\toprule
Method & Math500 & Minerva & Olympiad & AIME24 & AMC23 & $\Avg$ \\
\midrule
LoRA & 39.60 & 9.49 & 10.80 & 0.41 & 17.03 & 15.47 \\
LoRA +20k & 40.12 & 9.72 & \underline{11.58} & \underline{1.04} & \underline{18.91} & \underline{16.27} \\
SCALE w/ CE & \underline{41.97} & \underline{10.31} & 11.36 & 0.82 & 16.41 & 16.17 \\
\midrule
SCALE w/ Entropy & \textbf{71.60} & \textbf{31.44} & \textbf{31.34} & \textbf{8.55} & \textbf{46.25} & \textbf{37.84} \\
\bottomrule
\end{tabular*}
\end{table}


\begin{table*}[t]
\caption{Three-seed mathematical-reasoning results on \Qwen{}-1.5B (\%). Scores are Avg@16, reported as mean $\pm$ sample standard deviation.}
\label{tab:multiseed-math}
\centering
\small
\setlength{\tabcolsep}{3pt}
\renewcommand{\arraystretch}{1.1}
\resizebox{\textwidth}{!}{%
\begin{tabular}{@{}lcccccc@{}}
\toprule
Method & Math500 & Minerva & Olympiad & AIME24 & AMC23 & $\Avg$ \\
\midrule
LoRA & $39.52 \pm 0.07$ & $9.60 \pm 0.33$ & $10.59 \pm 0.39$ & $0.62 \pm 0.21$ & $16.72 \pm 1.43$ & $15.41 \pm 0.44$ \\
DFT & $\underline{65.30 \pm 1.03}$ & $\underline{21.42 \pm 1.89}$ & $\underline{27.22 \pm 0.37}$ & $\underline{7.64 \pm 1.78}$ & $\underline{39.95 \pm 0.65}$ & $\underline{32.31 \pm 0.92}$ \\
\midrule
\rowcolor{gray!15}
\Method{} & $\mathbf{71.73 \pm 0.57}$ & $\mathbf{27.18 \pm 2.47}$ & $\mathbf{31.56 \pm 0.31}$ & $\mathbf{10.21 \pm 1.08}$ & $\mathbf{48.44 \pm 1.18}$ & $\mathbf{37.83 \pm 0.20}$ \\
\bottomrule
\end{tabular}%
}
\end{table*}

\paragraph{Entropy provides a stronger calibration signal.}
Table~\ref{tab:gate-objective-ablation} compares cross-entropy (CE) and entropy gate calibration with pure LoRA continued on the same 20k examples. The gate uses signed tanh with range $[-6,3]$; continued LoRA has no gate. Entropy reaches 37.84 Math Avg@16, versus 16.27 for continued LoRA and 16.17 for CE.


\paragraph{Stable performance across training seeds.}
Across seeds 42, 3407, and 2025, \Method{} uses the fixed $[-6,3]$ gate range and achieves $37.83\pm0.20$ Math Avg@16 (Table~\ref{tab:multiseed-math}). It leads on all five benchmarks.

\vspace{-0.8\baselineskip}
\input{Sections/MechanismAnalysis}

\input{Sections/TrainingTime}

\section{Proofs and Structural Conditions for Sections 2--4}
\label{app:rw-proofs}
\subsection{Proof of the entropy-advantage identity}
\begin{proof}[Proof of Lemma~\ref{lem:rw-entropy}]
For a strictly positive finite softmax distribution,
\begin{align}
 \nabla H(p)
 &=-\sum_a p(a)(\log p(a)+1)\nabla\log p(a)\notag\\
 &=-\sum_a p(a)(\log p(a)+H(p))\nabla\log p(a),
 \label{eq:rw-proof-entropy}
\end{align}
because $\sum_a p(a)\nabla\log p(a)=\nabla\sum_a p(a)=0$.
Moreover, $\E_p[\log p+H(p)]=0$.
Since $\partial\log p(a)/\partial z(b)=\mathbf1[a=b]-p(b)$,
$\partial H/\partial z(b)=-p(b)(\log p(b)+H(p))$.
Applying the same computation to the softmax restricted to a locally fixed
support proves Eqs.~\eqref{eq:rw-topk-logit-gradient} and
\eqref{eq:rw-topk-policy}.
\end{proof}

\subsection{Local one-sidedness under weighted CE}
\begin{proof}[Proof of Proposition~\ref{prop:rw-onesided}]
Differentiating a weighted CE term with fixed weight gives
$-\nabla_z(w\ell)=w(e_y-p)$. Its $y$ coordinate is $w(1-p_y)\geq0$,
and its other coordinates are $-wp_a\leq0$. For any $a\neq y$, the direct
logit-margin change is proportional to $w(1-p_y+p_a)\geq0$.
Thus changing a nonnegative weight rescales a single direct CE direction.
\end{proof}
This property is not preserved as a coordinatewise sign guarantee after
shared-network propagation. If $K_t=\partial z_t/\partial\theta$, then
one parameter update on several tokens has first-order effect
\begin{equation}
 \Delta z_t
 =\frac{\eta}{|\mathcal T|}\sum_s
      w_s K_tK_s^\top(e_{y_s}-\pi_s)+o(\eta).
 \label{eq:rw-shared-ce}
\end{equation}
The cross-token matrices need not preserve signs. Neither this proposition
nor the on-policy representation rules out unlearning through parameter
sharing or through a different loss.

\subsection{Policy credit and the task-alignment identity}
At a fixed support, let $d_i=\partial z/\partial\lambda_i$.
The score derivative is
\begin{equation}
 \frac{\partial\log q(a)}{\partial\lambda_i}
 =d_i(a)-\E_q d_i.
 \label{eq:rw-centered-score}
\end{equation}
Since $A^{\mathrm{ent}}=z-\E_qz$, Lemma~\ref{lem:rw-entropy} yields
$-\partial H(q)/\partial\lambda_i=\Cov_q(z,d_i)$.
For fixed action values $r$, differentiation likewise gives
$\partial\E_qr/\partial\lambda_i=\Cov_q(r,d_i)$.
Averaging over all affected positions proves
Eq.~\eqref{eq:rw-gate-credit}. For shared parameters,
\begin{equation}
 \nabla_\psi V_k=\avg{\E_{q_t}[A_t^r s_t]},
 \qquad -\log k\,\nabla_\psi F_k
                  =\avg{\E_{q_t}[A_{k,t}^{\mathrm{ent}}s_t]}.
 \label{eq:rw-two-policy-gradients}
\end{equation}
Subtracting $\beta$ times the first identity from the second proves
Eq.~\eqref{eq:rw-alignment-identity}.

\paragraph{A structural sufficient condition.}
Write, on each current support,
\begin{equation}
 z_t(a)=\beta r_t(a)+b_t+\varepsilon_t(a),\qquad \beta>0.
 \label{eq:rw-signal-decomposition}
\end{equation}
Here $r_t$ is fixed independently of entropy. The decomposition alone is
always possible and imposes no restriction. Its useful condition is that the
delta error has small effect on the accessible directions. Indeed,
\begin{equation}
 A_{k,t}^{\mathrm{ent}}-\beta A_t^r
     =\varepsilon_t-\E_{q_t}\varepsilon_t,
 \qquad
 e_\psi=\avg{(J_t^\psi)^\top C(q_t)\varepsilon_t}.
 \label{eq:rw-projected-delta}
\end{equation}
The centering matrix is $C(q_t)=\operatorname{diag}(q_t)-q_tq_t^\top$,
restricted to the support. Large logit error in directions annihilated by
$(J_t^\psi)^\top C(q_t)$ does not corrupt the gate signal. Freezing the
feature operators restricts this accessible family; it does not automatically
make the delta small.

For an individual intervention coordinate, let
\makeatletter
\if@twocolumn
$h_i=\avg{\Cov_{q_t}(r_t,d_{t,i})}$ and
$b_i=\avg{\Cov_{q_t}(\varepsilon_t,d_{t,i})}$.
\else
\begin{equation*}
 h_i=\avg{\Cov_{q_t}(r_t,d_{t,i})},
 \qquad
 b_i=\avg{\Cov_{q_t}(\varepsilon_t,d_{t,i})}.
\end{equation*}
\fi
\makeatother
Then
\begin{equation}
 -\log k\,\frac{\partial F_k}{\partial\lambda_i}=\beta h_i+b_i.
 \label{eq:rw-coordinate-signal}
\end{equation}
If $h_i\neq0$ and $|b_i|<\beta|h_i|$, entropy gives the correct sign of task
contribution. The vector assumption in the main text instead guarantees
aggregate improvement with shared gates; it need not give a correct sign for
every coordinate separately.

\subsection{Proof of local task improvement and a finite step bound}
\begin{proof}[Proof of Theorem~\ref{thm:rw-improvement}]
Set $v=\nabla_\psi V_k$ and $G=\beta v+e_\psi$.
The step is $\psi^+=\psi+\eta G/\log k$.
By Cauchy--Schwarz and Assumption~\ref{ass:rw-alignment},
\begin{equation}
 \langle v,G\rangle
 \geq\beta\|v\|^2-\|v\|\|e_\psi\|
 \geq\beta(1-\rho)\|v\|^2>0.
 \label{eq:rw-positive-inner-product}
\end{equation}
The first-order expansion of $V_k$ establishes
Eq.~\eqref{eq:rw-value-improvement}; a sufficiently small step stays within
the differentiable neighborhood and makes the remainder smaller than the
strictly positive leading term.
\end{proof}
If $\nabla V_k$ is $L_V$-Lipschitz in that neighborhood, the stronger bound is
\begin{equation}
 V_k(\psi^+)-V_k(\psi)
 \geq\left[
 \frac{\eta\beta(1-\rho)}{\log k}
 -\frac{L_V\eta^2\beta^2(1+\rho)^2}{2(\log k)^2}
 \right]\|v\|^2.
 \label{eq:rw-finite-value-bound}
\end{equation}
It is positive for
$0<\eta<2(1-\rho)\log k/[L_V\beta(1+\rho)^2]$, provided the segment stays
in that neighborhood. This is a deterministic Euclidean-gradient statement,
not an unconditional guarantee for an optimizer with momentum or finite
minibatch noise.

The same proof with $G=\widehat G_k$ establishes the semi-gradient and
full-value variants in Appendix~\ref{app:rw-calculus}, replacing the delta
by $e_\psi+b_{\mathrm{sg}}$ or $e_{\mathrm{total}}$ respectively. For any actual
implemented update vector $u$, the elementary local task-improvement test is
$\langle\nabla V,u\rangle>0$. A low numerical entropy alone does not certify it.

\subsection{A margin model for helpful and harmful features}
In the binary setting, $p=\operatorname{sigmoid}(m)$ gives
\begin{equation}
 \frac{\mathrm dH}{\mathrm dp}=-m,
 \qquad \frac{\mathrm dp}{\mathrm dm}=p(1-p),
 \qquad \frac{\partial m}{\partial\lambda_i}=a_i,
 \label{eq:rw-binary-components}
\end{equation}
which proves Eq.~\eqref{eq:rw-margin}. For $m>0$, the sign of
$-\partial H/\partial\lambda_i$ equals the sign of
$\partial p/\partial\lambda_i$. With values $r(\mathrm{correct})=1$ and
$r(\mathrm{wrong})=0$, this is exactly task-value improvement.
At $m=0$ there is no first-order entropy signal; at $m<0$ it has the wrong
orientation. Thus a reliable semantic reference, not freezing alone, is needed.

For a multiclass distribution $p$, let a fixed feature tangent be
$d_i=a_i e_y+b_i\mathbf1$, where $y$ is the independently defined correct
action. Then
\begin{equation}
 \frac{\partial H}{\partial\lambda_i}
 =-a_i p_y\sum_{b\neq y}p_b(z_y-z_b),
 \qquad
 \frac{\partial p_y}{\partial\lambda_i}=a_i p_y(1-p_y).
 \label{eq:rw-multiclass-margin}
\end{equation}
The sign correspondence holds if
$z_y>\sum_{b\neq y}p_bz_b/(1-p_y)$.
A strictly highest correct logit suffices. This model concerns features that
change correct-versus-incorrect evidence in the specified form; arbitrary
multiclass deltas require the directional condition in
Eq.~\eqref{eq:rw-coordinate-signal}.
For language modeling, a demonstration token is not automatically the unique
correct action. Values can instead be assigned by a fixed external criterion
or a fixed reference continuation policy, but the fixed-prefix guarantee then
remains local and does not by itself prove improved free-running accuracy.

\subsection{Signed parameterization}
For $v=\tanh(g/\tau)$, the two branches of
Eq.~\eqref{eq:rw-transform} are
\begin{equation}
 \lambda(v)=
 \begin{cases}
  1+(1-L)v,&v<0,\\
  1+(U-1)v,&v\geq0.
 \end{cases}
 \label{eq:rw-transform-branches}
\end{equation}
Both slopes are positive when $L<0<1<U$; the branches agree at zero and
approach $L$ and $U$ as $v\to-1$ and $v\to1$. The image is $(L,U)$ with
closure $[L,U]$. Moreover, $\lambda=0$ at $v=-1/(1-L)$; more negative values
reverse the delta, and every positive $v$ extrapolates it.
At $g=0$, $\lambda=1$ exactly. The map is generally not differentiable at
that join, as its one-sided slopes in $g$ differ.

For fixed input $h$, the module-level displacement obeys
\begin{equation}
 \|u(h;\lambda)-u(h;1)\|
 \leq\max\{1-L,U-1\}\,\|\Delta Wh\|.
 \label{eq:rw-module-bound}
\end{equation}
This controls local delta scaling, not global output KL or task loss.
If $\Delta Wh\neq0$, expanding from $[0,1]$ to $(L,U)$ strictly expands that
module's attainable delta contributions. Degenerate deltas or downstream
cancellation need not enlarge the final prediction set. Training-time token
weights define a different object and are not covered by this set inclusion.

\subsection{Proof of the attenuation-boundary result}
\begin{proof}[Proof of Proposition~\ref{prop:rw-boundary}]
Every $\bar\lambda\in[0,1]^M$ lies in the interior of $[L,U]^M$.
For $g=\nabla F(\bar\lambda)\neq0$ and sufficiently small $\epsilon>0$,
$\lambda'=\bar\lambda-\epsilon g$ belongs to $(L,U)^M$. Differentiability gives
\begin{equation}
 F(\lambda')=F(\bar\lambda)-\epsilon\|g\|^2+o(\epsilon)
            <F(\bar\lambda).
 \label{eq:rw-boundary-proof}
\end{equation}
At a box-constrained local minimum, $g_i\geq0$ if $\bar\lambda_i=0$,
$g_i=0$ if $0<\bar\lambda_i<1$, and $g_i\leq0$ if $\bar\lambda_i=1$.
A nonzero gradient therefore identifies at least one blocked outward
direction: reversal from zero or extrapolation from one.
\end{proof}
If $\bar\lambda$ is a global minimum over the attenuation box, the proposition
also gives a strictly lower attainable objective in the extended box.
It does not imply endpoint optimality for nonlinear entropy. Its realization
by a shared gate requires a parameter direction $v$ with
$\langle\nabla_\lambda F,Bv\rangle<0$.
Moreover, the sign must be evaluated at the relevant boundary: a positive
derivative at one only supports attenuation initially, not an immediate jump
to a negative coefficient. A zero derivative supplies no first-order verdict
and does not exclude curvature, parameter-sharing, or subsequent-iterate effects.

\section{Limitation}
Despite the promising results, this work has several limitations that motivate future research. First, the current evaluation is limited to relatively small-scale language models. Our experiments are conducted on models up to 7B parameters, and we have not yet validated SCALE on larger-scale models (e.g., 30B+ parameters). Although the proposed gated delta calibration mechanism is parameter-efficient and theoretically independent of model size, its effectiveness and scalability on larger models remain to be investigated. Future work should evaluate whether the observed benefits persist in larger-scale settings with more complex adaptation dynamics. 

Second, our experiments are currently restricted to models from the Qwen family. Although we evaluate multiple generations and variants of Qwen models, including Qwen2.5-Math and Qwen3-Base, further validation on models from other families is necessary. In particular, evaluating SCALE on architectures and pretraining distributions beyond Qwen would provide a stronger assessment of its generality and robustness across different foundation models. 

Third, the current gate selection mechanism is not fully adaptive. In SCALE, the gate range is selected through empirical evaluation, and our experiments demonstrate that different gate ranges can effectively control the degree of delta suppression, reversal, and extrapolation. However, automatically determining the optimal intervention range remains an open problem. Developing more principled adaptive gate selection strategies, potentially based on task signals, uncertainty estimation, or learned optimization criteria, is an important direction for future research.

%% file: Sections/MechanismAnalysis.tex
\subsection{Realized gate behavior}
\label{app:mechanism-analysis}

The mathematical-reasoning gates use $[-6,3]$ for \Qwen{}-1.5B and $[-4,4]$ for \Qwen{}-7B. The code-generation gates on \Qwen{}-7B use $[0,6]$ and $[-3,3]$. The multiplier $\lambda$ scales the frozen LoRA delta: negative values reverse it, values between zero and one suppress or preserve it, and values above one extrapolate it.

\begin{table}[ht]
\caption{Distribution of gate multipliers on generated mathematical-reasoning tokens (\%).}
\label{tab:math-gate-distribution}
\centering
\small
\setlength{\tabcolsep}{5pt}
\renewcommand{\arraystretch}{1.12}
\begin{tabular*}{\linewidth}{@{\extracolsep{\fill}}lrrr@{}}
\toprule
Backbone & $\lambda<0$ & $0\leq\lambda\leq1$ & $\lambda>1$ \\
\midrule
\Qwen{}-1.5B & 39.36 & 17.81 & 42.82 \\
\Qwen{}-7B & 44.34 & 19.11 & 36.55 \\
\bottomrule
\end{tabular*}
\end{table}

\begin{table}[ht]
\caption{Mathematical-reasoning gate interventions. Scores are Math Avg@16; parentheses show the change from the original gate in percentage points.}
\label{tab:math-gate-intervention}
\centering
\small
\setlength{\tabcolsep}{5pt}
\renewcommand{\arraystretch}{1.12}
\begin{tabular*}{\linewidth}{@{\extracolsep{\fill}}lrr@{}}
\toprule
Gate applied at inference & \Qwen{}-1.5B & \Qwen{}-7B \\
\midrule
Original $\lambda$ & 37.05 & 42.52 \\
No reversal, $\max(\lambda,0)$ & 20.10 ($-16.95$) & 33.37 ($-9.15$) \\
No extrapolation, $\min(\lambda,1)$ & 30.06 ($-6.99$) & 30.06 ($-12.46$) \\
Unit interval, $\operatorname{clip}(\lambda,0,1)$ & 19.44 ($-17.61$) & 24.07 ($-18.45$) \\
\bottomrule
\end{tabular*}
\end{table}

Table~\ref{tab:math-gate-distribution} shows that the learned gates realize both reversal and extrapolation during generation. The two Qwen2.5-Math backbones assign $\lambda<0$ to 39.36\% and 44.34\% of generated-token coefficients, respectively, and $\lambda>1$ to 42.82\% and 36.55\%. Thus, most generated-token coefficients lie outside the unit interval in both models.

The forward interventions in Table~\ref{tab:math-gate-intervention} show that removing reversal lowers Math Avg@16 by 16.95 and 9.15 points on the 1.5B and 7B models. Removing extrapolation lowers it by 6.99 and 12.46 points. Constraining the gate to $[0,1]$ produces the largest reduction for each model. Both realized reversal and extrapolation contribute to mathematical-reasoning accuracy in these checkpoints.

\begin{table}[H]
\caption{Gate interventions for code generation on \Qwen{}-7B. Scores are HumanEval pass@1 (\%).}
\label{tab:code-gate-intervention}
\centering
\small
\setlength{\tabcolsep}{5pt}
\renewcommand{\arraystretch}{1.12}
\begin{tabular*}{\linewidth}{@{\extracolsep{\fill}}llrr@{}}
\toprule
Gate range & Gate applied at inference & pass@1 & Change \\
\midrule
$[0,6]$ & Original $\lambda$ & 57.11 & --- \\
          & Unit interval, $\operatorname{clip}(\lambda,0,1)$ & 47.15 & $-9.96$ \\
\midrule
$[-3,3]$ & Original $\lambda$ & 49.80 & --- \\
           & No reversal, $\max(\lambda,0)$ & 52.64 & $+2.85$ \\
           & No extrapolation, $\min(\lambda,1)$ & 45.12 & $-4.67$ \\
           & Unit interval, $\operatorname{clip}(\lambda,0,1)$ & 47.97 & $-1.83$ \\
\bottomrule
\end{tabular*}
\end{table}

Table~\ref{tab:code-gate-intervention} shows a 9.96-point decrease when the $[0,6]$ gate is restricted to the unit interval. With the $[-3,3]$ gate, removing extrapolation lowers pass@1 by 4.67 points, while removing reversal raises it by 2.85 points. The code results favor extrapolation rather than reversal in these checkpoints.

%% file: Sections/TrainingTime.tex
\subsection{Training Time}
\label{app:training-time}

Figure~\ref{fig:single-gpu-training-time} compares training time on one H200
GPU for the two Qwen2.5-Math backbones on mathematical reasoning and code
generation. SCALE denotes Stage-II gate training, averaged across five gate
ranges with one run per range; LoRA, DFT, DoRA, and TALR denote their
standalone Stage-I training runs. Full SCALE training also includes its
preceding LoRA stage. Excluding queueing and evaluation, gate training takes
20.0--46.0 minutes across the four settings, compared with 53.0--305.1
minutes for the baseline runs.

\begin{figure}[H]
\centering
\includegraphics[width=0.85\textwidth]{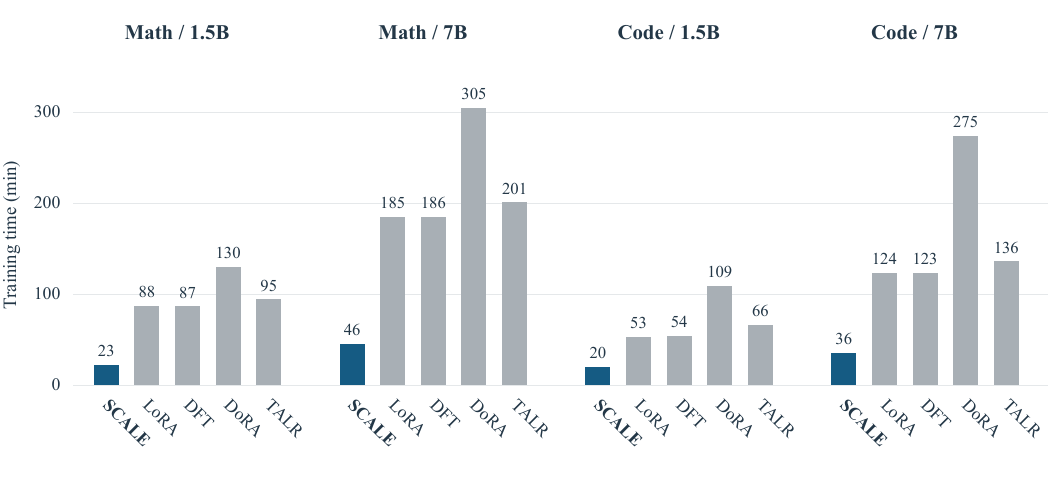}
\caption{Single-GPU training time (minutes) for mathematical reasoning and
code generation on Qwen2.5-Math-1.5B and 7B.}
\label{fig:single-gpu-training-time}
\end{figure}
\FloatBarrier